\documentclass[11pt]{article} 

\usepackage{times}
\usepackage{amsmath,amsthm,amsfonts}
\usepackage{color}
\usepackage{algorithm}
\usepackage[noend]{algorithmic}
\usepackage{graphicx}
\usepackage{psfrag}
\usepackage{natbib}
\usepackage{epstopdf}
\usepackage{subfigure}
\usepackage{fullpage}
\usepackage{hyperref}

\newcommand{\R}{\mathbb{R}}

\newcommand{\E}{\mathbb{E}}

\newcommand\supp{\operatorname{supp}}

\newcommand{\name}{{strict saddle}}
\newcommand{\bigname}{{Strict saddle}}

\newcommand{\Hess}{\nabla^2}

\newcommand{\tlO}{\tilde{O}}
\newcommand{\tlOmega}{\tilde{\Omega}}
\newcommand{\nameCQ}{ $\alpha_c$-RLICQ }
\newcommand{\BigC}{\mathfrak{S}}

\newtheorem{theorem}{Theorem}
\newtheorem{lemma}[theorem]{Lemma}
\newtheorem*{remark}{Remark}
\newtheorem*{claim}{Remark}

\theoremstyle{definition}
\newtheorem{definition}[theorem]{Definition}

\title{Escaping From Saddle Points \--- \\  Online Stochastic Gradient for Tensor Decomposition}

\date{}
\author{Rong Ge\thanks{Microsoft Research New England, rongge@microsoft.com} \and
Furong Huang \thanks{University of California Irvine, Department of Electrical Engineering and Computer Science, furongh@uci.edu} \and
Chi Jin \thanks{University of California Berkeley, Department of Electrical Engineering and Computer Science, chijin@cs.berkeley.edu} \and
Yang Yuan \thanks{Cornell University, Computer Science Department, yangyuan@cs.cornell.edu}}

\begin{document}

\maketitle

\begin{abstract}
We analyze stochastic gradient descent for optimizing non-convex functions. In many cases for non-convex functions the goal is to find a reasonable local minimum, and the main concern is that gradient updates are trapped in {\em saddle points}. In this paper we identify {\em \name} property for non-convex problem that allows for efficient optimization. Using this property we show that stochastic gradient descent converges to a local minimum in a polynomial number of iterations. To the best of our knowledge this is the first work that gives {\em global} convergence guarantees for stochastic gradient descent on non-convex functions with exponentially many local minima and saddle points. 

Our analysis can be applied to orthogonal tensor decomposition, which is widely used in learning a rich class of  latent variable models. We propose a new optimization formulation for the tensor decomposition problem that has \name~property. As a result we get the first online algorithm for orthogonal tensor decomposition with global convergence guarantee.

\end{abstract}



\section{Introduction}

Stochastic gradient descent is one of the basic algorithms in optimization. It is often used to solve the following stochastic optimization problem
\begin{equation}
w = \arg\min_{w\in \R^d} f(w), \textrm{~where~} f(w) = \E_{x\sim \mathcal{D}}[\phi(w,x)]
\label{eq:opt}
\end{equation}
Here $x$ is a data point that comes from some unknown distribution $\mathcal{D}$, and $\phi$ is a {\em loss function} that is defined for a pair $(x,w)$. We hope to minimize the expected loss $\E[\phi(w,x)]$.

When the function $f(w)$ is convex, convergence of stochastic gradient descent is well-understood \citep{ICML2012Rakhlin_261, shalev2009stochastic}. However, stochastic gradient descent is not only limited to convex functions. Especially, in the context of neural networks, stochastic gradient descent is known as the ``backpropagation'' algorithm~\citep{rumelhart1988learning}, and has been the main algorithm that underlies the success of deep learning~\citep{bengio2009learning}. However, the guarantees in the convex setting does not transfer to the non-convex settings.

Optimizing a non-convex function is NP-hard in general. The difficulty comes from two aspects. First, a non-convex function may have many local minima, and it might be hard to find the best one (global minimum) among them. Second, even finding a local minimum might be hard as there can be many saddle points which have $0$-gradient but are not local minima\footnote{See Section~\ref{sec:sgd} for definition of saddle points.}. In the most general case, there is no known algorithm that guarantees to find a local minimum in polynomial number of steps. The discrete analog (finding local minimum in domains like $\{0,1\}^n$) has been studied in complexity theory and is PLS-complete~\citep{johnson1988easy}.

In many cases, especially in those related to deep neural networks~\citep{dauphin2014identifying}\\ \citep{choromanska2014loss}, the main bottleneck in optimization is not due to local minima, but the existence of many saddle points. Gradient based algorithms are in particular susceptible to saddle point problems as they only rely on the gradient information. The saddle point problem is alleviated for second-order methods that also rely on the Hessian information~\citep{dauphin2014identifying}.

However, using Hessian information usually increases the memory requirement and computation time per iteration. As a result many applications still use stochastic gradient and empirically get reasonable results. In this paper we investigate why stochastic gradient methods can be effective even in presence of saddle point, in particular we answer the following question:

\medskip
{\noindent \textbf{Question:}} Given a non-convex function $f$ with many saddle points, what properties of $f$ will guarantee stochastic gradient descent to converge to a local minimum efficiently?
\medskip

We identify a property of non-convex functions which we call {\em \name}. Intuitively, this property guarantees local progress if we have access to the Hessian information. Surprisingly we show with only first order (gradient) information, stochastic gradient can escape the saddle points efficiently. 
We give a framework for analyzing stochastic gradient in both unconstrained and equality-constrained case using this property.

We apply our framework to {\em orthogonal tensor decomposition}, which is a core problem in learning many latent variable models (see discussion in~\ref{sec:prelim:tensor}). 
The tensor decomposition problem is inherently susceptible to the saddle point issues, as the problem asks to find $d$ different components and any permutation of the true components yields a valid solution. Such symmetry creates exponentially many local minima and saddle points in the optimization problem.  Using our new analysis of stochastic gradient, we give the first online algorithm for orthogonal tensor decomposition with global convergence guarantee. This is a key step towards making tensor decomposition algorithms more scalable.

\subsection{Summary of Results}

\paragraph{\bigname~functions}

Given a function $f(w)$ that is twice differentiable, we call $w$ a stationary point if $\nabla f(w) = 0$. A stationary point can either be a local minimum, a local maximum or a saddle point. 
We identify an interesting class of non-convex functions which we call \name. For these functions the Hessian of every saddle point has a negative eigenvalue. 
In particular, this means that local second-order algorithms which are similar to the ones in \citep{dauphin2014identifying} can always make some progress. 

It may seem counter-intuitive why stochastic gradient can work in these cases: in particular if we run the basic gradient descent starting from a stationary point then it will not move. However, we show that the saddle points are not stable and that the randomness in stochastic gradient helps the algorithm to escape from the saddle points.
%
%
%
%

\begin{theorem}[informal] Suppose $f(w)$ is \name~(see Definition~\ref{def:robustcondition}), Noisy Gradient Descent (Algorithm~\ref{algo:sgdwn}) outputs a point that is close to a local minimum in polynomial number of steps.
\end{theorem}


\paragraph{Online tensor decomposition} Requiring all saddle points to have a negative eigenvalue may seem strong, but it already allows non-trivial applications to natural non-convex optimization problems. 
As an example, we consider the orthogonal tensor decomposition problem. This problem is the key step in spectral learning for many latent variable models (see more discussions in Section~\ref{sec:prelim:tensor}).

We design a new objective function for tensor decomposition that is \name. 
\begin{theorem}
Given random samples $X$ such that $T = \E[g(X)] \in \R^{d^4}$ is an orthogonal $4$-th order tensor (see Section~\ref{sec:prelim:tensor}), there is an objective function $f(w) = \E[\phi(w,X)]$ $w\in \R^{d\times d}$ such that every local minimum of $f(w)$ corresponds to a valid decomposition of $T$. Further, function $f$ is \name.
\end{theorem}

Combining this new objective with our framework for analyzing stochastic gradient in non-convex setting,  we get the first online algorithm for orthogonal tensor decomposition with global convergence guarantee.

\subsection{Related Works}

\paragraph{Relaxed notions of convexity} In optimization theory and economics, there are extensive works on understanding functions that behave similarly to convex functions (and in particular can be optimized efficiently). Such notions involve pseudo-convexity~\citep{mangasarian1965pseudo}, quasi-convexity~\\\citep{kiwiel2001convergence}, invexity\citep{hanson1999invexity} and their variants. More recently there are also works that consider classes that admit more efficient optimization procedures like RSC (restricted strong convexity)~\citep{agarwal2010fast}. Although these classes involve functions that are non-convex, the function (or at least the function restricted to the region of analysis) still has a unique stationary point that is the desired local/global minimum. Therefore these works cannot be used to prove global convergence for problems like tensor decomposition, where by symmetry of the problem there are multiple local minima and saddle points.

\paragraph{Second-order algorithms} The most popular second-order method is the Newton's method. Although Newton's method converges fast near a local minimum, its global convergence properties are less understood in the more general case. For non-convex functions,~\citep{frieze1996learning} gave a concrete example where second-order method converges to the desired local minimum in polynomial number of steps (interestingly the function of interest is trying to find one component in a $4$-th order orthogonal tensor, which is a simpler case of our application). As Newton's method often converges also to saddle points, to avoid this behavior, different trusted-region algorithms are applied~\citep{dauphin2014identifying}.

\paragraph{Stochastic gradient and symmetry} The tensor decomposition problem we consider in this paper has the following symmetry: the solution is a set of $d$ vectors $v_1,...,v_d$. If $(v_1,v_2,...,v_d)$ is a solution, then for any permutation $\pi$ and any sign flips $\kappa \in \{\pm 1\}^d$, $(.., \kappa_i v_{\pi(i)}, ...)$ is also a valid solution. In general, symmetry is known to generate saddle points, and variants of gradient descent often perform reasonably in these cases (see~\citep{saad1995line},~\citep{rattray1998natural},~\citep{inoue2003line}). The settings in these work are different from ours, and none of them give bounds on number of steps required for convergence.

There are many other problems that have the same symmetric structure as the tensor decomposition problem, including the sparse coding problem~\citep{olshausen1997sparse} and many deep learning applications~\citep{bengio2009learning}. In these problems the goal is to learn multiple ``features'' where the solution is invariant under permutation. Note that there are many recent papers on iterative/gradient based algorithms for problems related to matrix factorization~\citep{jain2013low,saxe2013exact}. These problems
often have very different symmetry, as if $Y = AX$ then for any invertible matrix $R$ we know $Y = (AR)(R^{-1} X)$. In this case all the equivalent solutions are in a connected low dimensional manifold and there need not be saddle points between them.

\vspace*{-0.2in}
\section{Preliminaries}

\paragraph{Notation} Throughout the paper we use $[d]$ to denote set $\{1,2,...,d\}$. We use $\|\cdot\|$ to denote the $\ell_2$ norm of vectors and spectral norm of matrices. For a matrix we use $\lambda_{min}$ to denote its smallest eigenvalue. For a function $f:\R^d\to \R$,  $\nabla f$ and $\nabla^2 f$ denote its gradient vector and Hessian matrix.

\vspace*{-0.1in}
\subsection{Stochastic Gradient Descent}

The stochastic gradient aims to solve the stochastic optimization problem (\ref{eq:opt}), which we restate here:
$$
w = \arg\min_{w\in \R^d} f(w), \textrm{~where~} f(w) = \E_{x\sim \mathcal{D}}[\phi(w,x)].
$$
Recall $\phi(w,x)$ denotes the loss function evaluated for sample $x$ at point $w$.
The algorithm follows a stochastic gradient 
\begin{equation}
w_{t+1} = w_{t} - \eta \nabla_{w_t} \phi(w_t,x_t),
\end{equation}
where $x_t$ is a random sample drawn from distribution $\mathcal{D}$ and $\eta$ is the {\em learning rate}.


In the more general setting, stochastic gradient descent can be viewed as optimizing an arbitrary function $f(w)$ given a stochastic gradient oracle.

\begin{definition}
For a function $f(w):\R^d \to \R$, a function $SG(w)$ that maps a variable to a random vector in $\R^d$ is a stochastic gradient oracle if $\E[SG(w)] = \nabla f(w)$ and $\|SG(w) - \nabla f(w)\| \le Q$.
\end{definition}

In this case the update step of the algorithm becomes $w_{t+1} = w_t - \eta SG(w_t)$.
%

\paragraph{Smoothness and Strong Convexity}
Traditional analysis for stochastic gradient often assumes the function is smooth and strongly convex. A function is $\beta$-smooth if for any two points $w_1,w_2$,
\begin{equation}
\|\nabla f(w_1) - \nabla f(w_2)\| \le \beta\|w_1-w_2\|.
\end{equation}
When $f$ is twice differentiable this is equivalent to assuming that the spectral norm of the Hessian matrix is bounded by $\beta$.
We say a function is $\alpha$-strongly convex if the Hessian at any point has smallest eigenvalue at least $\alpha$ ($\lambda_{min}(\Hess f(w)) \ge \alpha$).

Using these two properties, previous work~\citep{ICML2012Rakhlin_261} shows that stochastic gradient converges at a rate of $1/t$. In this paper we consider non-convex functions, which can still be $\beta$-smooth but cannot be strongly convex.

\paragraph{Smoothness of Hessians} 
%
We also require the Hessian of the function $f$ to be smooth.  We say a function $f(w)$ has $\rho$-Lipschitz Hessian if for any two points $w_1,w_2$ we have
\begin{equation}
\|\Hess f(w_1) - \Hess f(w_2)\| \le \rho \|w_1-w_2\|.
\end{equation}
This is a third order condition that is 
true if the third order derivative exists and is bounded.
%

\subsection{Tensors decomposition}

\label{sec:prelim:tensor}
A $p$-th order tensor is a $p$-dimensional array. In this paper we will mostly consider $4$-th order tensors.
If $T\in \R^{d^4}$ is a $4$-th order tensor, we use $T_{i_1,i_2,i_3,i_4} (i_1,...,i_4\in [d])$ to denote its $(i_1,i_2,i_3,i_4)$-th entry.

Tensors can be constructed from tensor products. We use $(u\otimes v)$ to denote a $2$nd order tensor where $(u\otimes v)_{i,j} = u_iv_j$. This generalizes to higher order and we use $u^{\otimes 4}$ to denote the $4$-th order tensor
$$
[u^{\otimes 4}]_{i_1,i_2,i_3,i_4} = u_{i_1}u_{i_2}u_{i_3}u_{i_4}.
$$ 
We say a $4$-th order tensor $T\in \R^{d^4}$ has an {\em orthogonal decomposition} if it can be written as
\begin{equation}
T = \sum_{i=1}^d a_i^{\otimes 4}, \label{eq:orthodecomp}
\end{equation}
where $a_i$'s are orthonormal vectors that satisfy $\|a_i\| = 1$ and $a_i^T a_j = 0$ for $i\ne j$. We call the vectors $a_i$'s the components of this decomposition. Such a decomposition is unique up to permutation of $a_i$'s and sign-flips.

A tensor also defines a multilinear form (just as a matrix defines a bilinear form), for a $p$-th order tensor $T\in \R^{d^p}$ and matrices $M_i\in \R^{d\times n_i} i\in[p]$, we define
$$
[T(M_1,M_2,...,M_p)]_{i_1,i_2,...,i_p} = \sum_{j_1,j_2,...,j_p\in[d]} T_{j_1,j_2,...,j_p} \prod_{t\in[p]} M_t[i_t,j_t].
$$
That is, the result of the multilinear form $T(M_1,M_2,...,M_p)$ is another tensor in $\R^{n_1\times n_2\times \cdots \times n_p}$. We will most often use vectors or identity matrices in the multilinear form. In particular, for a $4$-th order tensor $T\in \R^{d^4}$ we know $T(I,u,u,u)$ is a vector and $T(I,I,u,u)$ is a matrix. In particular, if $T$ has the orthogonal decomposition in (\ref{eq:orthodecomp}), we know $T(I,u,u,u) = \sum_{i=1}^d (u^T a_i)^3 a_i$ and $T(I,I,u,u) = \sum_{i=1}^d (u^Ta_i)^2 a_ia_i^T$.

Given a tensor $T$ with an orthogonal decomposition, the orthogonal tensor decomposition problem asks to find the individual components $a_1,...,a_d$. This is a central problem in learning many latent variable models, including Hidden Markov Model, multi-view models, topic models, mixture of Gaussians and Independent Component Analysis (ICA). See the discussion and citations in \cite{JMLR:v15:anandkumar14b}. Orthogonal tensor decomposition problem can be solved by many algorithms even when the input is a noisy estimation $\tilde{T} \approx T$ ~\citep{harshman1970foundations,kolda2001orthogonal,JMLR:v15:anandkumar14b}. 
 In practice this approach has been successfully applied to ICA~\citep{comon2002tensor}, topic models~\citep{zou2013contrastive} and community detection~\citep{huang2013fast}.

\section{Stochastic gradient descent for \name~function}

\label{sec:sgd}

In this section we discuss the properties of saddle points, and show if all the saddle points are well-behaved then stochastic gradient descent finds a local minimum for a non-convex function in polynomial time.

\subsection{\bigname~property}
\label{subsec:strictsaddleproperty}
For a twice differentiable function $f(w)$, we call the points stationary points if their gradients are equal to $0$. Stationary points could be local minima, local maxima or saddle points. By local optimality conditions~\citep{wright1999numerical}, in many cases we can tell what type a point $w$ is by looking at its Hessian: if $\Hess f(w)$ is positive definite then $w$ is a local minimum; if $\Hess f(w)$ is negative definite then $w$ is a local maximum; if $\Hess f(w)$ has both positive and negative eigenvalues then $w$ is a saddle point. These criteria do not cover all the cases as there could be degenerate scenarios: $\Hess f(w)$ can be positive semidefinite with an eigenvalue equal to 0, in which case the point could be a local minimum or a saddle point.

If a function does not have these degenerate cases, then we say the function is \name:

\begin{definition}
A twice differentiable function $f(w)$ is {\em \name}, if all its local minima have $\Hess f(w) \succ 0$ and all its other stationary points satisfy $\lambda_{min} (\Hess f(w)) < 0$.
\end{definition}

Intuitively, if we are not at a stationary point, then we can always follow the gradient and reduce the value of the function. If we are at a saddle point, we need to consider a second order Taylor expansion:
$$
f(w+\Delta w) \approx w + (\Delta w)^T \Hess f(w) (\Delta w) + O(\|\Delta w\|^3).
$$
Since the \name~property guarantees $\Hess f(w)$ to have a negative eigenvalue, there is always a point that is near $w$ and has strictly smaller function value. It is possible to make local improvements as long as we have access to second order information. However it is not clear whether the more efficient stochastic gradient updates can work in this setting. 

To make sure the local improvements are significant, we use a robust version of the \name~property:

\begin{definition}
\label{def:robustcondition}
A twice differentiable function $f(w)$ is $(\alpha, \gamma, \epsilon, \delta)$-{\em\name}, if for any point $w$ at least one of the following is true
\begin{enumerate}
\item $\|\nabla f(w)\| \ge \epsilon$.
\item $\lambda_{min}(\Hess f(w)) \le -\gamma$.
\item There is a local minimum $w^\star$ such that $\|w-w^\star\| \le \delta$, and the function $f(w')$ restricted to $2\delta$ neighborhood of $w^\star$ ($\|w'-w^\star\|\le 2\delta$) is $\alpha$-strongly convex.
\end{enumerate}
\end{definition}

Intuitively, this condition says for any point whose gradient is small, it is either close to a robust local minimum, or is a saddle point (or local maximum) with a significant negative eigenvalue. 


\begin{algorithm}[ht]
 \caption{Noisy Stochastic Gradient}
 \label{algo:sgdwn}
 \begin{algorithmic}[1]
 \REQUIRE Stochastic gradient oracle $SG(w)$, initial point $w_0$, desired accuracy $\kappa$.
  \ENSURE $w_t$ that is close to some local minimum $w^\star$.
  \STATE Choose $\eta = \min\{\tilde{O}(\kappa^2 / \log (1/\kappa)), \eta_{\max}\}$, $T = \tilde{O}(1/\eta^2)$
 	\FOR{$t = 0$ to $T-1$}
 	\STATE Sample noise $n$ uniformly from unit sphere.
	\STATE $w_{t+1} \leftarrow w_{t} - \eta (SG(w) + n)$
 		\ENDFOR
 \end{algorithmic}
 \end{algorithm}

We purpose a simple variant of stochastic gradient algorithm, where the only difference to the traditional algorithm is we add an extra noise term to the updates. The main benefit of this additional noise is that we can guarantee there is noise in every direction, which allows the algorithm to effectively explore the local neighborhood around saddle points. If the noise from stochastic gradient oracle already has nonnegligible variance in every direction, our analysis also applies without adding additional noise. We show noise can help the algorithm escape from saddle points and optimize \name~functions.

\begin{theorem} [Main Theorem]\label{thm:sgdmain}
Suppose a function $f(w):\R^d\to \R$ that is $(\alpha, \gamma, \epsilon, \delta)$-\name, and has a stochastic gradient oracle with radius at most $Q$. Further, suppose the function is bounded by $|f(w)| \le B$, is $\beta$-smooth and has $\rho$-Lipschitz Hessian. 
Then there exists a threshold $\eta_{\max} = \tilde{\Theta}(1)$, so that for any $\zeta>0$, and
for any $\eta \le \eta_{\max} / \max\{1, \log(1/\zeta)\}$, with probability at least $1-\zeta$ in $t = \tlO(\eta^{-2}\log (1/\zeta))$ iterations, Algorithm~\ref{algo:sgdwn} (Noisy Gradient Descent) outputs a point $w_t$ that is $\tlO(\sqrt{\eta\log(1/\eta\zeta)})$-close to some local minimum $w^\star$.
\end{theorem}

Here (and throughout the rest of the paper) $\tlO(\cdot)$ ($\tilde{\Omega},\tilde{\Theta}$) hides the factor that is polynomially dependent on all other parameters (including $Q$, $1/\alpha$, $1/\gamma$, $1/\epsilon$, $1/\delta$, $B$, $\beta$, $\rho$, and $d$), but independent of $\eta$ and $\zeta$. So it focuses on the dependency on $\eta$ and $\zeta$. 
Our proof technique can give explicit dependencies on these parameters however we hide these dependencies for simplicity of presentation. 


\begin{remark} [Decreasing learning rate]
Often analysis of stochastic gradient descent uses decreasing learning rates and the algorithm converges to a local (or global) minimum. Since the function is strongly convex in the small region close to local minimum, we can use Theorem \ref{thm:sgdmain} to first find a point that is close to a local minimum, and then apply standard analysis of SGD in the strongly convex case (where we decrease the learning rate by $1/t$ and get $1/\sqrt{t}$ convergence in $\|w-w^\star\|$).
\end{remark}
%

In the next part we sketch the proof of the main theorem. Details are deferred to Appendix~\ref{sec:unconstrained}.

\subsection{Proof sketch}

In order to prove Theorem~\ref{thm:sgdmain}, we analyze the three cases in Definition~\ref{def:robustcondition}. When the gradient is large, we show the function value decreases in one step (see Lemma~\ref{lem:gradient}); when the point is close to a local minimum, we show with high probability it cannot escape in the next polynomial number of iterations (see Lemma~\ref{lem:minimum}). 

\begin{lemma}[Gradient]
\label{lem:gradient}
Under the assumptions of Theorem~\ref{thm:sgdmain}, for any point with $\|\nabla f(w_t)\| \ge C\sqrt{\eta}$ (where $C = \tilde{\Theta}(1)$) and $C\sqrt{\eta} \le \epsilon$, after one iteration we have $\E[f(w_{t+1})] \le f(w_t) - \tlOmega(\eta^2)$.
\end{lemma}

The proof of this lemma is a simple application of the smoothness property.

\begin{lemma}[Local minimum]
\label{lem:minimum}
Under the assumptions of Theorem~\ref{thm:sgdmain}, for any point $w_t$ that is $\tlO(\sqrt{\eta}) < \delta$ close to local minimum $w^\star$, in $\tlO(\eta^{-2}\log (1/\zeta))$ number of steps all future $w_{t+i}$'s are $\tlO(\sqrt{\eta\log(1/\eta\zeta)})$-close with probability at least $1-\zeta/2$.
\end{lemma}

The proof of this lemma is similar to the standard analysis \citep{ICML2012Rakhlin_261} of stochastic gradient descent in the smooth and strongly convex setting, except we only have local strongly convexity. The proof appears in Appendix~\ref{sec:unconstrained}.


The hardest case is when the point is ``close'' to a saddle point: it has gradient smaller than $\epsilon$ and smallest eigenvalue of the Hessian bounded by $-\gamma$. In this case we show the noise in our algorithm helps the algorithm to escape:

\begin{lemma}[Saddle point]
\label{lem:saddle}
Under the assumptions of Theorem~\ref{thm:sgdmain}, for any point $w_t$ where $\|\nabla f(w_t)\| \le C\sqrt{\eta}$ (for the same $C$ as in Lemma~\ref{lem:gradient}), and $\lambda_{\min}(\Hess f(w_t)) \le -\gamma$, there is a number of steps $T$ that depends on $w_t$ such that $\E[f(w_{t+T})] \le f(w_t)-\tlOmega(\eta)$. The number of steps $T$ has a fixed upper bound $T_{max}$ that is independent of $w_t$ where $T \le T_{max} = \tilde{O}(1/\eta)$.
\end{lemma}

Intuitively, at point $w_t$ there is a good direction that is hiding in the Hessian. The hope of the algorithm is that the additional (or inherent) noise in the update step makes a small step towards the correct direction, and then the gradient information will reinforce this small perturbation and the future updates will ``slide'' down the correct direction. 

To make this more formal, we consider a coupled sequence of updates $\tilde{w}$ such that the function to minimize is just the local second order approximation $$\tilde{f}(w) = f(w_t) + \nabla f(w_t)^T (w-w_t) + \frac{1}{2}(w-w_t)^T\Hess f(w_t) (w-w_t).$$


	The dynamics of stochastic gradient descent for this quadratic function is easy to analyze as 
	$\tilde{w}_{t+i}$ can be calculated analytically. Indeed, we show the expectation of 
	$\tilde{f}(\tilde{w})$ will decrease. We then use the smoothness of the function to show that as long as the points did not go very far from $w_t$, the two update sequences $\tilde{w}$ and $w$ will remain close to each other, and thus $\tilde{f}(\tilde{w}_{t+i}) \approx f(w_{t+i})$. Finally we prove the future $w_{t+i}$'s (in the next $T$ steps) will remain close to $w_t$ with high probability by Martingale bounds. The detailed proof appears in Appendix~\ref{sec:unconstrained}.

With these three lemmas it is easy to prove the main theorem. Intuitively, as long as there is a small probability of being $\tlO(\sqrt{\eta})$-close to a local minimum, we can always apply Lemma~\ref{lem:gradient} or Lemma~\ref{lem:saddle} to make the expected function value decrease by $\tlOmega(\eta)$ in at most $\tlO(1/\eta)$ iterations, this cannot go on for more than $\tlO(1/\eta^2)$ iterations because in that case the expected function value will decrease by more than $2B$, but $\max f(x) - \min f(x) \le 2B$ by our assumption. Therefore in $\tlO(1/\eta^2)$ steps with at least constant probability $w_t$ will become $\tilde{O}(\sqrt{\eta})$-close to a local minimum. By Lemma~\ref{lem:minimum} we know once it is close it will almost always stay close, so we can repeat this $\log (1/\zeta)$ times to get the high probability result. More details appear in Appendix~\ref{sec:unconstrained}.
%
%
%
%
\subsection{Constrained Problems}
\label{sec:constrainedproblem}

In many cases, the problem we are facing are constrained optimization problems. In this part we briefly describe how to adapt the analysis to problems with equality constraints (which suffices for the tensor application). Dealing with general inequality constraint is left as future work.

For a constrained optimization problem:
\begin{align}
&\min_{w\in \R^d} \quad \quad  f(w) \\
&\text{s.t.} \quad \quad c_i(w) = 0, \quad \quad i\in[m]\nonumber
\end{align}
in general we need to consider the set of points in a low dimensional manifold that is defined by the constraints. In particular, in the algorithm after every step we need to project back to this manifold (see Algorithm~\ref{algo:psgdwn} where $\Pi_\mathcal{W}$ is the projection to this manifold).


\begin{algorithm}[ht]
 \caption{Projected Noisy Stochastic Gradient}
 \label{algo:psgdwn}
 \begin{algorithmic}[1]
 \REQUIRE Stochastic gradient oracle $SG(w)$, initial point $w_0$, desired accuracy $\kappa$.
  \ENSURE $w_t$ that is close to some local minimum $w^\star$.
  \STATE Choose $\eta = \min\{\tilde{O}(\kappa^2/\log (1/\kappa)), \eta_{\max}\}$, $T = \tilde{O}(1/\eta^2)$
 	\FOR{$t = 0$ to $T-1$}
 	\STATE Sample noise $n$ uniformly from unit sphere.
	\STATE $v_{t+1} \leftarrow w_{t} - \eta (SG(w) + n)$
	\STATE $w_{t+1} = \Pi_{\mathcal{W}}(v_{t+1})$
 		\ENDFOR
 \end{algorithmic}
 \end{algorithm}

For constrained optimization it is common to consider the Lagrangian:
\begin{equation}
\mathcal{L}(w, \lambda) =  f(w) - \sum_{i=1}^m \lambda_i c_i(w).
\end{equation}


Under common regularity conditions, it is possible to compute the value of the Lagrangian multipliers: $$\lambda^*(w)=\arg\min_{\lambda} \|\nabla_w \mathcal{L}(w, \lambda)\|.$$ We can also define the tangent space, which contains all directions that are orthogonal to all the gradients of the constraints: $\mathcal{T}(w) =\{v: \nabla c_i(w)^T v = 0; ~ i=1, \cdots, m \}$. In this case the corresponding gradient and Hessian we consider are the first-order and second-order partial derivative of Lagrangian $\mathcal{L}$
at point $(w, \lambda^*(w))$:
\begin{align}
&\chi(w) = \nabla_w \mathcal{L}(w, \lambda) |_{(w, \lambda^*(w))} =\nabla f(w) - \sum_{i=1}^m \lambda^*_i(w) \nabla c_i(w)  \\
&\mathfrak{M}(w) = \Hess_{ww} \mathcal{L}(w, \lambda) |_{(w, \lambda^*(w))} = \Hess f(w) - \sum_{i=1}^m \lambda^*_i(w) \nabla^2 c_i(w) 
\end{align}

We replace the gradient and Hessian with $\chi(w)$ and $\mathfrak{M}(w)$, and when computing eigenvectors of $\mathfrak{M}(w)$ we focus on its projection on the tangent space. In this way, we can get a similar definition for \name ~(see Appendix~\ref{sec:constrained}), and the following theorem.



\begin{theorem}(informal)\label{thm:constrainedinformal}
Under regularity conditions and smoothness conditions, if a constrained optimization problem satisfies \name~property, then for a small enough $\eta$, in $\tlO(\eta^{-2}\log 1/\zeta)$ iterations Projected Noisy Gradient Descent (Algorithm~\ref{algo:psgdwn}) outputs a point $w$ that is $\tlO(\sqrt{\eta}\log (1/\eta\zeta))$ close to a local minimum with probability at least $1-\zeta$.
\end{theorem}

Detailed discussions and formal version of this theorem are deferred to Appendix~\ref{sec:constrained}.



\section{Online Tensor Decomposition}\label{sec:tensors}

In this section we describe how to apply our stochastic gradient descent analysis to tensor decomposition problems. We first give a new formulation of tensor decomposition as an optimization problem, and show that it satisfies the \name~property. Then we explain how to compute stochastic gradient in a simple example of Independent Component Analysis (ICA)~\citep{icabook}.

\subsection{Optimization problem for tensor decomposition}

Given a tensor $T\in \R^{d^4}$ that has an orthogonal decomposition
\begin{equation}
T = \sum_{i=1}^d a_i^{\otimes 4},
\end{equation}
where the components $a_i$'s are orthonormal vectors ($\|a_i\| = 1$, $a_i^Ta_j = 0$ for $i\ne j$), the goal of orthogonal tensor decomposition is to find the components $a_i$'s.

This problem has inherent symmetry: for any permutation $\pi$ and any set of $\kappa_i\in \{\pm 1\},i\in[d]$, we know $u_i = \kappa_i a_{\pi(i)}$ is also a valid solution. This symmetry property makes the natural optimization problems non-convex.

In this section we will give a new formulation of orthogonal tensor decomposition as an optimization problem, and show that this new problem satisfies the \name~property.

Previously, \cite{frieze1996learning} solves the problem of finding one component, with the following objective function
\begin{equation}\label{eq:findone}
\max\limits_{\|u\|^2 = 1}  \quad T(u,u,u,u).
\end{equation}
In Appendix \ref{sec:warmup}, as a warm-up example we show this function is indeed \name, and we can apply Theorem~\ref{thm:constrainedinformal} to prove global convergence of stochastic gradient descent algorithm.

It is possible to find all components of a tensor by iteratively finding one component, and do careful {\em deflation}, as described in \cite{JMLR:v15:anandkumar14b} or \cite{arora2012provable}. However, in practice the most popular approaches like Alternating Least Squares \citep{comon2009tensor} or FastICA \citep{hyvarinen1999fast} try to use a single optimization problem to find all the components. Empirically these algorithms are often more robust to noise and model misspecification.

The most straight-forward formulation of the problem aims to minimize the {\em reconstruction error}
\begin{equation}\label{eq:reconstruction}
\min\limits_{\forall i, \|u_i\|^2 = 1} \quad \| T - \sum_{i=1}^d u_i^{\otimes 4}\|_F^2.
\end{equation}
Here $\|\cdot \|_F$ is the Frobenius norm of the tensor which is equal to the $\ell_2$ norm when we view the tensor as a $d^4$ dimensional vector. However, it is not clear whether this function satisfies the \name~property, and empirically stochastic gradient descent is unstable for this objective.

We propose a new objective that aims to minimize the correlation between different components:
\begin{equation}\label{eq:hardprob}
\min\limits_{\forall i, \|u_i\|^2 = 1}  \quad \sum_{i\ne j} T(u_i,u_i,u_j,u_j),
\end{equation}
To understand this objective intuitively, we first expand vectors $u_k$ in the orthogonal basis formed by $\{a_i\}$'s. That is, we can write $u_k = \sum_{i=1}^{d}z_k(i) a_i$, where $z_k(i)$ are scalars that correspond to the coordinates in the $\{a_i\}$ basis. In this way we can rewrite $T(u_k,u_k,u_l,u_l) = \sum_{i=1}^{d} (z_k(i))^2 (z_l(i))^2$. From this form it is clear that the $T(u_k,u_k,u_l,u_l)$ is always nonnegative, and is equal to $0$  only when the support of $z_k$ and $z_l$ do not intersect. For the objective function, we know in order for it to be equal to 0 the $z$'s must have disjoint support. Therefore, we claim that $\{u_k\}, \forall k\in[d]$ is equivalent to $\{a_i\}, \forall i\in[d]$ up to permutation and sign flips when the global minimum (which is 0) is achieved. 

We further show that this optimization program satisfies the \name~property and all its local minima in fact achieves global minimum value. The proof is deferred to Appendix \ref{sec:hardcase}.
\begin{theorem}
The optimization problem (\ref{eq:hardprob}) is $(\alpha, \gamma, \epsilon,\delta)$-\name, for $\alpha = 1$ and $\gamma,\epsilon,\delta = 1/\mbox{poly}(d)$. Moreover, all its local minima have the form $u_i = \kappa_i a_{\pi(i)}$ for some $\kappa_i = \pm 1$ and permutation $\pi(i)$.
\end{theorem}

\subsection{Implementing stochastic gradient oracle}
\label{sec:icagrad}
To design an online algorithm based on objective function \eqref{eq:hardprob}, we need to give an implementation for the stochastic gradient oracle.

In applications, the tensor $T$ is oftentimes the expectation of multilinear operations of samples $g(x)$ over $x$ where $x$ is generated from some distribution $\mathcal{D}$. In other words, for any $x\sim \mathcal{D}$, the tensor is $T =\E[g(x)] $. Using the linearity of the multilinear map, we know $\E[g(x)] (u_i,u_i,u_j,u_j) = \E[g(x)(u_i,u_i,u_j,u_j)]$. Therefore we can define the loss function $\phi(u,x) = \sum_{i\ne j} g(x)(u_i,u_i,u_j,u_j)$, and the stochastic gradient oracle $SG(u) = \nabla_u \phi(u,x)$.

For concreteness, we look at a simple ICA example. In the simple setting we consider an unknown signal $x$ that is uniform\footnote{In general ICA the entries of $x$ are independent, non-Gaussian variables.} in $\{\pm 1\}^d$, and an unknown orthonormal linear transformation\footnote{In general (under-complete) ICA this could be an arbitrary linear transformation, however usually after the ``whitening'' step (see \cite{cardoso1989source}) the linear transformation becomes orthonormal.} $A$ ($AA^T = I$). The sample we observe is $y :=  Ax \in \R^d$. Using standard techniques (see \cite{cardoso1989source}), we know the $4$-th order cumulant of the observed sample is a tensor that has orthogonal decomposition. Here for simplicity we don't define 4-th order cumulant, instead we give the result directly.

Define tensor $Z\in \R^{d^4}$ as follows:
\begin{equation*}
\begin{array}{ll}
Z(i,i,i,i) =3, &   \forall i\in [d] \\
Z(i,i,j,j) = Z(i,j,i,j) = Z(i,j,j,i) = 1, &\forall i\ne j\in [d]\\
\end{array}
\end{equation*}
where all other entries of $Z$ are equal to $0$. 
The tensor $T$ can be written as a function of the auxiliary tensor $Z$ and multilinear form of the sample $y$.
\begin{lemma}\label{lm:constructTensorZ}
The expectation $\E[\frac{1}{2}(Z - y^{\otimes 4})] = \sum_{i=1}^d a_i^{\otimes 4}=T$, where $a_i$'s are columns of the unknown orthonormal matrix $A$.
\end{lemma}

This lemma is easy to verify, and is closely related to cumulants~\citep{cardoso1989source}.  Recall that $\phi(u,y)$ denotes the loss (objective) function evaluated at sample $y$ for point $u$. Let $\phi(u,y) = \sum_{i\ne j} \frac{1}{2}(Z - y^{\otimes 4})(u_i,u_i,u_j,u_j)$. By Lemma~\ref{lm:constructTensorZ}, we know that  $\E[\phi(u,y)]$ is equal to the objective function as in Equation~\eqref{eq:hardprob}.
Therefore we rewrite objective (\ref{eq:hardprob}) as the following stochastic optimization problem
\begin{equation*}
\min\limits_{\forall i, \|u_i\|^2 = 1} \quad  \E[\phi(u,y)] ,~\text{where}~ \phi(u,y) = \sum_{i\ne j} \frac{1}{2}(Z - y^{\otimes 4})(u_i,u_i,u_j,u_j)
\end{equation*}
The stochastic gradient oracle is then 
\begin{equation}\label{eq:icasg}
\nabla_{u_i} \phi(u,y) = \sum\limits_{j\neq i}\left(\left\langle u_j ,u_j \right\rangle u_i  + 2 \left\langle u_i ,u_j \right\rangle u_j  - \left\langle u_j , y\right\rangle^2 \left\langle u_i ,y\right\rangle y \right).
\end{equation}
Notice that computing this stochastic gradient does not require constructing the $4$-th order tensor $T - y^{\otimes 4}$. In particular, this stochastic gradient can be computed very efficiently: 
\begin{claim}
The stochastic gradient (\ref{eq:icasg}) can be computed in $O(d^3)$ time for one sample or $O(d^3+d^2k)$ for average of $k$ samples. 
\end{claim}

\begin{proof} The proof is straight forward as the first two terms take $O(d^3)$ and is shared by all samples. The third term can be efficiently computed once the inner-products between all the $y$'s and all the $u_i$'s are computed (which takes $O(kd^2)$ time).
\end{proof}

\section{Experiments}\label{sec:experi}

We run simulations for Projected Noisy Gradient Descent (Algorithm~\ref{algo:psgdwn}) applied to orthogonal tensor decomposition.
The results show that the algorithm converges from random initial points efficiently (as predicted by the theorems), and our new formulation (\ref{eq:hardprob}) performs better than reconstruction error (\ref{eq:reconstruction}) based formulation.

\paragraph{Settings} We set dimension $d = 10$, the input tensor $T$ is a random tensor in $\R^{10^4}$ that has orthogonal decomposition (\ref{eq:orthodecomp}). The step size is chosen carefully for respective objective functions. The performance is measured by normalized reconstruction error $\mathcal{E} =\left({\|T - \sum_{i=1}^{d} u_i ^{\otimes4}\|_F^2}\right)/{\| T\|_F^2}$. 

\paragraph{Samples and stochastic gradients}
We use two ways to generate samples and compute stochastic gradients. In the first case we generate  sample $x$ by setting it equivalent to $d^{\frac{1}{4}} a_i$ with probability $1/d$. It is easy to see that $\E[x^{\otimes 4}] = T$. This is a very simple way of generating samples, and we use it as a sanity check for the objective functions.

In the second case we consider the ICA example introduced in Section~\ref{sec:icagrad}, and use Equation (\ref{eq:icasg}) to compute a stochastic gradient. In this case the stochastic gradient has a large variance, so we use mini-batch of size 100 to reduce the variance.

\paragraph{Comparison of objective functions} We use the simple way of generating samples for our new objective function (\ref{eq:hardprob}) and reconstruction error objective (\ref{eq:reconstruction}). The result is shown in Figure~\ref{fig:obj}. Our new objective function is empirically more stable (always converges within 10000 iterations); the reconstruction error do not always converge within the same number of iterations and often exhibits long periods with small improvement (which is likely to be caused by saddle points that do not have a significant negative eigenvalue).

\paragraph{Simple ICA example} As shown in Figure~\ref{fig:ICA}, our new algorithm also works in the ICA setting. When the learning rate is constant the error stays at a fixed small value. When we decrease the learning rate the error converges to 0.

\begin{figure}[!htb]
\hfill
\subfigure[New Objective (\ref{eq:hardprob})]{\psfrag{reconstruction error}[Bc]{\scriptsize  Reconstruction Error}\psfrag{iter}[c]{\scriptsize  Iteration}\includegraphics[width=0.4\textwidth]{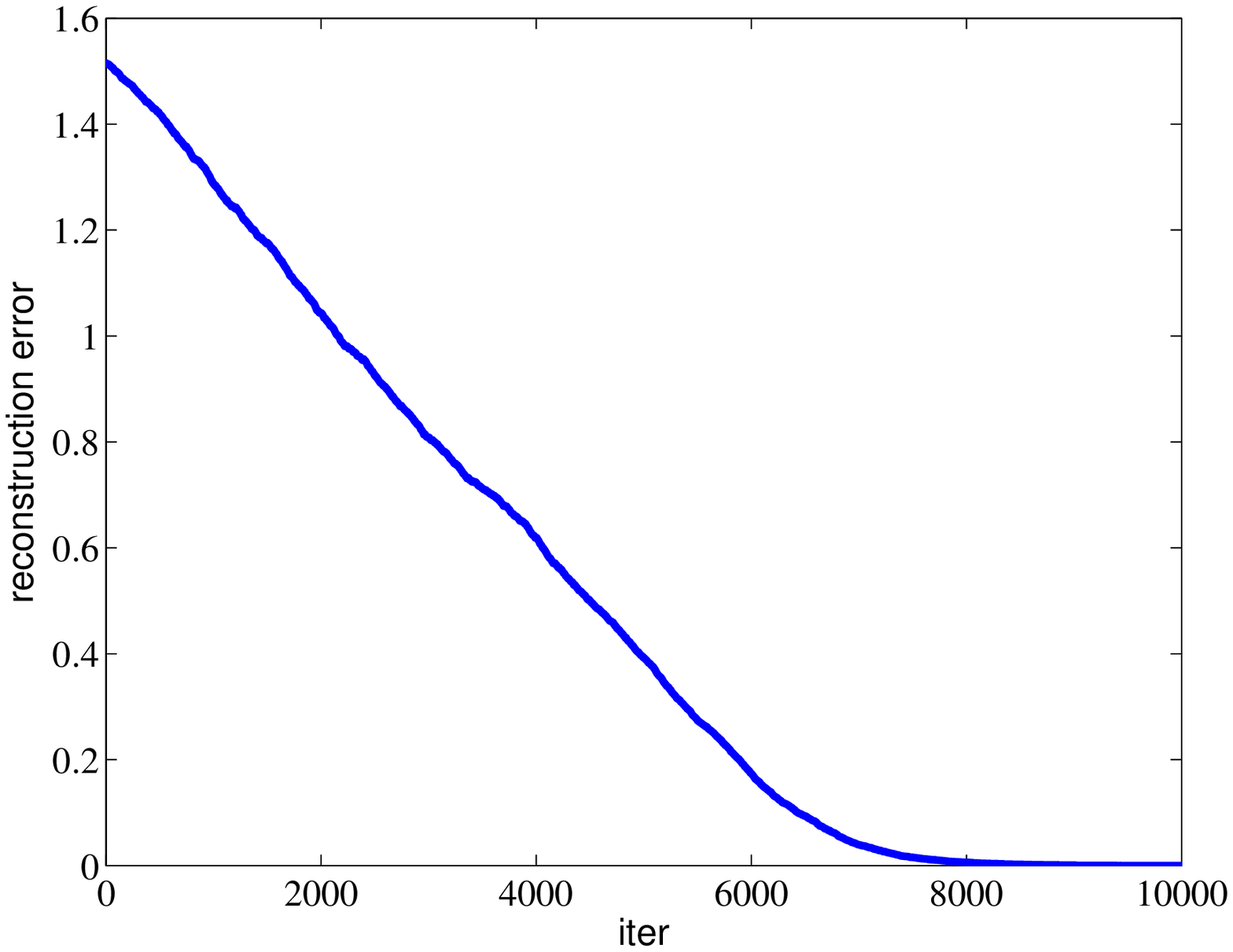}}\label{fig:NG1}
\hfill
\subfigure[Reconstruction Error Objective (\ref{eq:reconstruction})]{\psfrag{reconstruction error}[Bc]{\scriptsize  Reconstruction Error}\psfrag{iter}[c]{\scriptsize  Iteration}\includegraphics[width=0.4\textwidth]{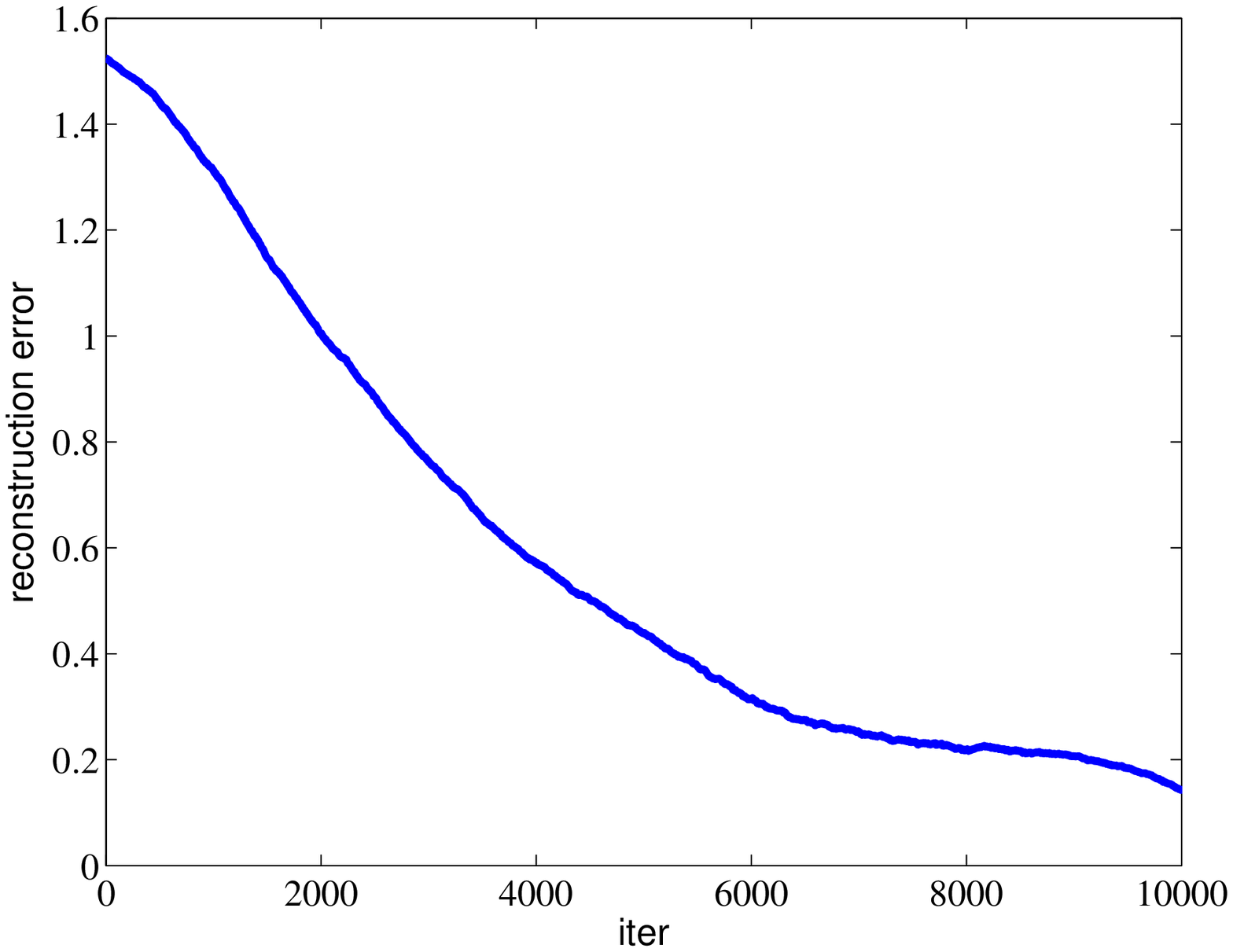}}\label{fig:IO1}
\hfill
\caption{Comparison of different objective functions}\label{fig:obj}
\end{figure}

\begin{figure}[!htb]
\hfill
\subfigure[Constant Learning Rate $\eta$]{\psfrag{reconstruction error}[Bc]{\scriptsize  Reconstruction Error}\psfrag{iter}[c]{\scriptsize  Iteration}\includegraphics[width=0.4\textwidth]{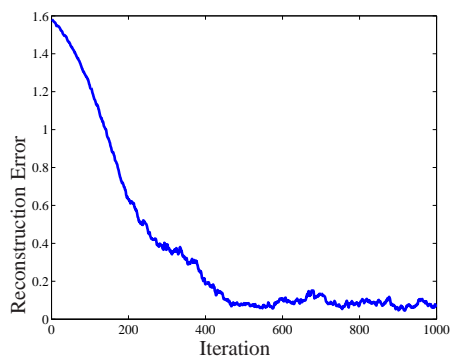}}\label{fig:ICA1}
\hfill
\subfigure[Learning Rate $\eta/t$ (in $\log$ scale)]{\psfrag{reconstruction error}[Bc]{\scriptsize  Reconstruction Error}\psfrag{iter}[c]{\scriptsize  Iteration}\includegraphics[width=0.4\textwidth]{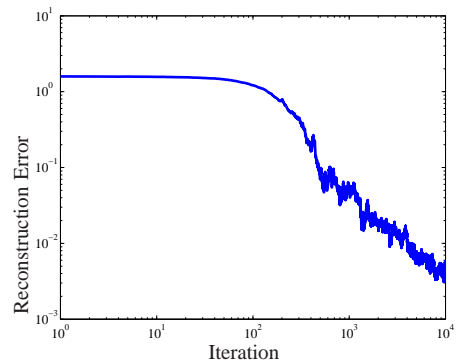}}\label{fig:ICA2}
\hfill
\caption{ICA setting performance with mini-batch of size 100}\label{fig:ICA}
\end{figure}

\section{Conclusion}
In this paper we identify the \name~property and show stochastic gradient descent converges to a local minimum under this assumption. This leads to new online algorithm for orthogonal tensor decomposition. We hope this is a first step towards understanding stochastic gradient for more classes of non-convex functions. We believe \name~property can be extended to handle more functions, especially those functions that have similar symmetry properties.


\clearpage

\bibliographystyle{apalike}
\bibliography{gradient}

\clearpage

\appendix

\section{Detailed Analysis for Section~\ref{sec:sgd} in Unconstrained Case}

\label{sec:unconstrained}

In this section we give detailed analysis for noisy gradient descent, under the assumption that the unconstrained problem satisfies $(\alpha,\gamma,\epsilon,\delta)$-\name~property.

The algorithm we investigate in Algorithm~\ref{algo:sgdwn}, we can combine the randomness in the stochastic gradient oracle and the artificial noise, and rewrite the update equation in form:
\begin{equation} \label{SGD_update}
	w_t = w_{t-1} - \eta (\nabla f(w_{t-1}) + \xi_{t-1})
\end{equation}
where $\eta$ is step size, $\xi = SG(w_{t-1}) - \nabla f(w_{t-1}) + n$ (recall $n$ is a random vector on unit sphere) is the combination of two source of noise.

By assumption, we know $\xi$'s are independent and they satisfying $\E \xi = 0$, 
$\|\xi\| \le Q+1$. Due to the explicitly added noise in Algorithm \ref{algo:sgdwn}, 
we further have   $\E \xi\xi^T \succ \frac{1}{d}I$.
For simplicity, we assume $ \E \xi\xi^T = \sigma^2I$, for some constant $\sigma = \tilde{\Theta}(1)$, 
then the algorithm we are running is exactly the same as Stochastic Gradient Descent (SGD).
Our proof can be very easily extended to the case when $\frac{1}{d} I \preceq\E [\xi\xi^T] \preceq (Q+\frac{1}{d})I$ because both the upper and lower bounds are $\tilde{\Theta}(1)$.


We first restate the main theorem in the context of stochastic gradient descent.

\begin{theorem} [Main Theorem]\label{thm:sgdmain_unconstraint}
Suppose a function $f(w):\R^d\to \R$ that is $(\alpha, \gamma, \epsilon, \delta)$-\name, and has a stochastic gradient oracle where the noise satisfy $\E \xi\xi^T = \sigma^2I$. Further, suppose the function is bounded by $|f(w)| \le B$, is $\beta$-smooth and has $\rho$-Lipschitz Hessian. Then there exists a threshold $\eta_{\max} = \tilde{\Theta}(1)$, so that for any $\zeta>0$, and
for any $\eta \le \eta_{\max} / \max\{1, \log (1/\zeta)\}$,
with probability at least $1-\zeta$ in $t = \tlO(\eta^{-2}\log( 1/\zeta))$ iterations, SGD outputs a point $w_t$ that is $\tlO(\sqrt{\eta\log(1/\eta\zeta)})$-close to some local minimum $w^\star$.
\end{theorem}

Recall that $\tlO(\cdot)$ ($\tilde{\Omega},\tilde{\Theta}$) hides the factor that is polynomially dependent on all other parameters, but independent of $\eta$ and $\zeta$. So it focuses on the dependency on $\eta$ and $\zeta$. Throughout the proof, we interchangeably use both $\mathcal{H}(w)$ and $\Hess f(w)$ to represent the Hessian matrix of $f(w)$.

As we discussed in the proof sketch in Section~\ref{sec:sgd}, we analyze the behavior of the algorithm in three different cases. The first case is when the gradient is large.

\begin{lemma} \label{thm::case1}
Under the assumptions of Theorem~\ref{thm:sgdmain_unconstraint}, for any point with $\|\nabla f(w_0)\| \ge \sqrt{2\eta\sigma^2\beta d}$ where $\sqrt{2\eta\sigma^2\beta d} < \epsilon$, after one iteration we have:
\begin{equation}
	\E f(w_1) - f(w_{0}) \le - \tlOmega(\eta^2)
\end{equation}
\end{lemma} 

\begin{proof}
Choose $\eta_{\max} < \frac{1}{\beta}$, then by update equation Eq.(\ref{SGD_update}), we have:
\begin{align}
	\E f(w_1) -  f(w_{0}) &\le \nabla f(w_{0})^T \E(w_1-w_{0}) + \frac{\beta}{2}\E\|w_1-w_{0}\|^2 \nonumber \\
	& = \nabla f(w_{0})^T \E\left (- \eta (\nabla f(w_{0}) + \xi_{0})\right )
	 + \frac{\beta}{2}\E\left \|- \eta (\nabla f(w_{0}) + \xi_{0})\right \|^2 \nonumber \\	
	& = -(\eta - \frac{\beta\eta^2}{2})\|\nabla f(w_{0})\|^2 + \frac{\eta^2 \sigma^2 \beta d}{2}\nonumber \\
	& \le -\frac{\eta}{2}\|\nabla f(w_{0})\|^2  + \frac{\eta^2\sigma^2 \beta d}{2}
	\le -\frac{\eta^2\sigma^2 \beta d}{2}
\end{align}
which finishes the proof.
\end{proof}

\begin{lemma}\label{thm::case3}
Under the assumptions of Theorem~\ref{thm:sgdmain_unconstraint}, for any initial point $w_0$ that is $\tlO(\sqrt{\eta}) < \delta$ close to a local minimum $w^\star$, 
with probability at least $1-\zeta/2$, we have following holds simultaneously:
\begin{equation}
 \forall t\le \tlO(\frac{1}{\eta^2}\log \frac{1}{\zeta}), \quad \|w_{t} - w^\star\| \le \tlO(\sqrt{\eta\log \frac{1}{\eta\zeta}})<\delta 
\end{equation}
where $w^\star$ is the locally optimal point.
\end{lemma}

\begin{proof}
We shall construct a supermartingale and use Azuma's inequality~\citep{azuma1967weighted} to prove this result.

Let filtration $\mathfrak{F}_t = \sigma\{\xi_0, \cdots \xi_{t-1}\}$, and note $\sigma\{\Delta_0, \cdots, \Delta_t \} \subset \mathfrak{F}_t$, where $\sigma\{\cdot\}$ denotes the sigma field.
Let event $\mathfrak{E}_t = \{\forall \tau \le t, \|w_{\tau} - w^\star\| \le \mu\sqrt{\eta\log \frac{1}{\eta\zeta}} < \delta \}$, where $\mu$ is independent of $(\eta, \zeta)$, and will be specified later. To ensure the correctness of proof, $\tilde{O}$ notation in this proof will never hide any dependence on $\mu$.
Clearly there's always a small enough choice of $\eta_{\max}= \tilde{\Theta}(1)$ to make $\mu\sqrt{\eta\log \frac{1}{\eta\zeta}} < \delta$ holds as long as 
$\eta \le \eta_{\max} / \max\{1, \log (1/\zeta)\}$.
Also note $\mathfrak{E}_t \subset \mathfrak{E}_{t-1}$, that is $1_{\mathfrak{E}_t} \le 1_{\mathfrak{E}_{t-1}}$.

By Definition~\ref{def:robustcondition} of 
$(\alpha, \gamma, \epsilon, \delta)$-\name, we know $f$ is locally $\alpha$-strongly convex
in the $2\delta$-neighborhood of $w^\star$.
Since $\nabla f(w^\star) = 0$, we have
\begin{align}
\nabla f(w_t)^T (w_t - w^\star)1_{\mathfrak{E}_{t}}  \ge \alpha \|w_t - w^\star\|^21_{\mathfrak{E}_{t}}
\end{align}

Furthermore, with  $\eta_{\max} < \frac{\alpha}{\beta^2}$, using $\beta$-smoothness, we have:
\begin{align}
\E[\|w_{t} - w^\star\|^21_{\mathfrak{E}_{t-1}}|\mathfrak{F}_{t-1}] = &\E[\|w_{t-1} - \eta (\nabla f(w_{t-1}) + \xi_{t-1}) - w^\star\|^2|\mathfrak{F}_{t-1}]1_{\mathfrak{E}_{t-1}} \nonumber \\
 = &\left[\|w_{t-1}- w^\star\|^2 - 2\eta \nabla f(w_{t-1})^T(w_{t-1}- w^\star) + \eta^2\|\nabla f(w_{t-1})\|^2 +\eta^2\sigma^2\right]1_{\mathfrak{E}_{t-1}}  \nonumber \\
  \le&[(1-2\eta \alpha + \eta^2 \beta^2)\|w_{t-1}- w^\star\|^2 + \eta^2\sigma^2]1_{\mathfrak{E}_{t-1}} \nonumber \\
 \le &[(1-\eta \alpha)\|w_{t-1}- w^\star\|^2 + \eta^2\sigma^2]1_{\mathfrak{E}_{t-1}}
\end{align}
Therefore, we have:
\begin{equation}
\left[\E[\|w_{t} - w^\star\|^2|\mathfrak{F}_{t-1}] - \frac{\eta}{ \alpha} \right]1_{\mathfrak{E}_{t-1}} \le (1-\eta \alpha) \left[\|w_{t-1} - w^\star\|^2 -\frac{\eta}{ \alpha} \right]1_{\mathfrak{E}_{t-1}}
\end{equation}

Then, let $G_t = (1-\eta\alpha)^{-t}(\|w_{t} - w^\star\|^2 - \frac{\eta}{ \alpha})$, we have:
\begin{equation}
\E [G_t1_{\mathfrak{E}_{t-1}} | \mathfrak{F}_{t-1}] \le G_{t-1}1_{\mathfrak{E}_{t-1}} \le G_{t-1}1_{\mathfrak{E}_{t-2}}
\end{equation}
which means $G_t1_{\mathfrak{E}_{t-1}}$ is a supermartingale.

Therefore, with probability 1, we have:
\begin{align}
&|G_t1_{\mathfrak{E}_{t-1}} -\E[G_{t}1_{\mathfrak{E}_{t-1}}|\mathfrak{F}_{t-1}] | \nonumber \\
\le & (1-\eta\alpha)^{-t}
[~ \|w_{t-1} - \eta \nabla f(w_{t-1}) - w^\star\|\cdot \eta\|\xi_{t-1}\| + \eta^2\|\xi_{t-1}\|^2 - \eta^2\sigma^2~]1_{\mathfrak{E}_{t-1}} \nonumber \\
\le & (1-\eta\alpha)^{-t}\cdot \tlO(\mu\eta^{1.5}\log^{\frac{1}{2}} \frac{1}{\eta\zeta}) = d_t
\end{align}

Let 
\begin{equation}
c_t = \sqrt{\sum_{\tau=1}^t d_\tau^2 } =  \tlO(\mu\eta^{1.5}\log^{\frac{1}{2}}  \frac{1}{\eta\zeta})\sqrt{\sum_{\tau=1}^t(1-\eta\alpha)^{-2\tau} } 
\end{equation}
By Azuma's inequality, with probability less than $\tlO(\eta^3\zeta)$,
we have:
\begin{align}
G_t1_{\mathfrak{E}_{t-1}}  > \tlO(1)c_t \log^{\frac{1}{2}} (\frac{1}{\eta\zeta}) + G_0
\end{align}

We know $G_t  > \tlO(1)c_t \log^{\frac{1}{2}}(\frac{1}{\eta\zeta}) + G_0$ is equivalent to:
\begin{align}
\|w_{t} - w^\star\|^2 > \tlO(\eta)
+ \tlO(1) (1-\eta\alpha)^{t}c_t \log^{\frac{1}{2}} (\frac{1}{\eta\zeta})
\end{align}
We know:
\begin{align}
&(1-\eta\alpha)^{t}c_t\log^{\frac{1}{2}} (\frac{1}{\eta\zeta}) =  \mu\cdot \tlO(\eta^{1.5}\log \frac{1}{\eta\zeta})\sqrt{\sum_{\tau=1}^t(1-\eta\alpha)^{2(t-\tau)} } \nonumber \\
= & \mu\cdot \tlO(\eta^{1.5}\log \frac{1}{\eta\zeta})\sqrt{\sum_{\tau=0}^{t-1}(1-\eta\alpha)^{2\tau} }  
\le   \mu\cdot \tlO(\eta^{1.5}\log \frac{1}{\eta\zeta})\sqrt{\frac{1}{1-(1-\eta\alpha)^2}} = \mu\cdot \tlO(\eta\log \frac{1}{\eta\zeta})
\end{align}
This means Azuma's inequality implies, there exist some $\tilde{C} = \tlO(1)$ so that:
\begin{align}
	P\left(\mathfrak{E}_{t-1} \cap\left\{\|w_{t} - w^\star\|^2 > \mu \cdot \tilde{C}\eta\log\frac{1}{\eta\zeta}) \right\} \right)
	\le \tlO(\eta^3\zeta)
\end{align}
By choosing $\mu > \tilde{C}$, this is equivalent to:
\begin{align}
	P\left(\mathfrak{E}_{t-1} \cap\left\{\|w_{t} - w^\star\|^2 > \mu^2\eta\log\frac{1}{\eta\zeta} \right\} \right)
	\le \tlO(\eta^3\zeta)
\end{align}
Then we have:
\begin{align}
	P(\overline{\mathfrak{E}}_{t} ) = 
	P\left(\mathfrak{E}_{t-1} \cap\left\{\|w_{t} - w^\star\| > \mu\sqrt{\eta\log\frac{1}{\eta\zeta}} \right\} \right) 	+ P(\overline{\mathfrak{E}}_{t-1} ) 
	\le  \tlO(\eta^3\zeta)+ P(\overline{\mathfrak{E}}_{t-1} )
\end{align}
By initialization conditions, we know $P(\overline{\mathfrak{E}}_{0} ) = 0$, and thus
$P(\overline{\mathfrak{E}}_{t} )
\le t \tlO(\eta^3\zeta)$. Take $t=\tlO(\frac{1}{\eta^2}\log\frac{1}{\zeta})$, 
we have $P(\overline{\mathfrak{E}}_{t} ) \le \tlO(\eta \zeta \log \frac{1}{\zeta})$.
When $\eta_{\max} = \tlO(1)$ is chosen small enough, and $\eta \le \eta_{\max}/\log(1/\zeta)$, this finishes the proof.
\end{proof}

\begin{lemma} \label{thm::case2}
Under the assumptions of Theorem~\ref{thm:sgdmain_unconstraint}, for any initial point $w_0$ where $\|\nabla f(w_0)\| \le \sqrt{2\eta\sigma^2\beta d} < \epsilon$, and $\lambda_{\min}(\mathcal{H}(w_0)) \le -\gamma$, then 
there is a number of steps $T$ that depends on $w_0$ such that:
\begin{equation}
	\E f(w_T) - f(w_0) \le - \tlOmega(\eta)
\end{equation}
The number of steps $T$ has a fixed upper bound $T_{max}$ that is independent of $w_0$ where $T \le T_{max} = O((\log d)/\gamma\eta)$.
\end{lemma} 

\begin{remark}
In general, if we relax the assumption $\E \xi\xi^T  = \sigma^2I$ to 
$ \sigma_{\min}^2I\preceq\E \xi\xi^T \preceq \sigma_{\max}^2I$, the upper bound $T_{max}$ of number of steps required in 
Lemma \ref{thm::case2} would be increased to $T_{max} = O(\frac{1}{\gamma\eta}(\log d+ \log \frac{\sigma_{\max}}{\sigma_{\min}}))$
\end{remark}
As we described in the proof sketch, the main idea is to consider a coupled update sequence that correspond to the local second-order approximation of $f(x)$ around $w_0$. We characterize this sequence of update in the next lemma.
%
%

%
%
%
%

\begin{lemma} \label{lem::case_Gaussian}
Under the assumptions of Theorem~\ref{thm:sgdmain_unconstraint}.
Let $\tilde{f}$ defined as local second-order approximation of $f(x)$ around $w_0$:
\begin{equation}\label{def_tilde_f}
\tilde{f}(w) \doteq f(w_0) + \nabla f(w_0)^T (w-w_0) + \frac{1}{2}(w-w_0)^T\mathcal{H}(w_0)(w-w_0)
\end{equation}
$\{\tilde{w}_t\}$ be the corresponding sequence generated by running SGD on function $\tilde{f}$, with $\tilde{w}_0 = w_0$.
For simplicity, denote $\mathcal{H} = \mathcal{H}(w_0)= \Hess f(w_0)$,
then we have analytically:
\begin{align}
&\nabla \tilde{f}(\tilde{w}_t)= (1-\eta\mathcal{H})^t\nabla f(w_0) -\eta \mathcal{H}\sum_{\tau=0}^{t-1}(1-\eta\mathcal{H})^{t-\tau-1}\xi_{\tau}\\
 &\tilde{w}_{t} - w_0 = -\eta \sum_{\tau = 0}^{t-1}(1-\eta \mathcal{H})^\tau \nabla f(w_0) -\eta
	\sum_{\tau=0}^{t-1}(1-\eta\mathcal{H})^{t-\tau-1}\xi_{\tau}  \label{dif_x}
\end{align}

Furthermore, for any initial point $w_0$ where $\|\nabla f(w_0)\| \le \tlO(\eta) < \epsilon$, and $\lambda_{\min}(\mathcal{H}(w_0)) = -\gamma_0$. 
Then, there exist a $T \in \mathbb{N}$ satisfying:
\begin{equation}\label{choose_t}
 \frac{d}{\eta\gamma_0} \le \sum_{\tau=0}^{T-1}(1+\eta \gamma_0)^{2\tau} < \frac{3d}{\eta\gamma_0}
\end{equation}
with probability at least $1-\tlO(\eta^3)$, we have
following holds simultaneously for all $t\le T$:
\begin{equation}
\|\tilde{w}_t - w_0\| \le \tlO(\eta^{\frac{1}{2}}\log \frac{1}{\eta}); 
\quad\quad
\|\nabla\tilde{f}(\tilde{w}_t)\| \le \tlO(\eta^{\frac{1}{2}}\log \frac{1}{\eta})
\end{equation}
\end{lemma}

\begin{proof}
Denote $\mathcal{H} = \mathcal{H}(w_0)$, since $\tilde{f}$ is quadratic, clearly we have:
\begin{equation}\label{derivative_tilde_recursive}
\nabla \tilde{f}(\tilde{w}_t) = \nabla \tilde{f}(\tilde{w}_{t-1}) + \mathcal{H} (\tilde{w}_t - \tilde{w}_{t-1})
\end{equation}
Substitute the update equation of SGD in Eq.(\ref{derivative_tilde_recursive}), we have:
\begin{align} \label{derivative_tilde}
\nabla \tilde{f}(\tilde{w}_t) &= \nabla \tilde{f}(\tilde{w}_{t-1}) - \eta\mathcal{H} ( \nabla \tilde{f}(\tilde{w}_{t-1}) + \xi_{t-1})\nonumber \\
				&= (1-\eta\mathcal{H})\nabla \tilde{f}(\tilde{w}_{t-1}) -\eta \mathcal{H}\xi_{t-1}  \nonumber\\
				&= (1-\eta\mathcal{H})^2\nabla \tilde{f}(\tilde{w}_{t-2}) -\eta \mathcal{H}\xi_{t-1}
				-\eta  \mathcal{H}(1-\eta\mathcal{H})\xi_{t-2}=\cdots\nonumber\\
				& = (1-\eta\mathcal{H})^t\nabla f(w_0) - \eta \mathcal{H}\sum_{\tau=0}^{t-1}(1-\eta\mathcal{H})^{t-\tau-1}\xi_{\tau}
\end{align}
Therefore, we have:
\begin{align} \label{dif_tilde}
\tilde{w}_{t} - w_0 &= -\eta\sum_{\tau=0}^{t-1} (\nabla \tilde{f}(\tilde{w}_\tau) + \xi_\tau) \nonumber \\
&= -\eta\sum_{\tau=0}^{t-1}\left  (
(1-\eta\mathcal{H})^\tau \nabla f(w_0) - \eta \mathcal{H}\sum_{\tau'=0}^{\tau-1}(1-\eta\mathcal{H})^{\tau -\tau'-1}\xi_{\tau'}
 + \xi_\tau\right ) \nonumber \\
 &= -\eta \sum_{\tau = 0}^{t-1}(1-\eta \mathcal{H})^\tau \nabla f(w_0) -\eta
	\sum_{\tau=0}^{t-1}(1-\eta\mathcal{H})^{t-\tau-1}\xi_{\tau} 
\end{align}

Next, we prove the existence of $T$ in Eq.(\ref{choose_t}).
Since $\sum_{\tau=0}^{t}(1+\eta \gamma_0)^{2\tau}$ is monotonically increasing w.r.t $t$, 
and diverge to infinity as $t \rightarrow \infty$. We know there is always some 
$T \in \mathbb{N}$ gives $\frac{d}{\eta\gamma_0} \le \sum_{\tau=0}^{T-1}(1+\eta \gamma_0)^{2\tau}$.
Let $T$ be the smallest integer satisfying above equation.
By assumption, we know $\gamma \le \gamma_0 \le L$, and
\begin{equation}
\sum_{\tau=0}^{t+1}(1+\eta \gamma_0)^{2\tau}
= 1 + (1+\eta \gamma_0)^2\sum_{\tau=0}^{t}(1+\eta \gamma_0)^{2\tau}
\end{equation}
we can choose $\eta_{\max} < \min\{(\sqrt{2}-1)/L, 2d/\gamma\}$ so that
\begin{equation}
\frac{d}{\eta\gamma_0} \le \sum_{\tau=0}^{T-1}(1+\eta \gamma_0)^{2\tau} 
\le 1 + \frac{2d}{\eta\gamma_0} \le \frac{3d}{\eta\gamma_0} 
\end{equation}

Finally, by Eq.(\ref{choose_t}), we know $T = O(\log d/\gamma_0\eta)$, and $(1+\eta \gamma_0)^T \le \tlO(1)$. Also because $\E\xi = 0$ and $\|\xi\| \le Q= \tlO(1)$ with probability 1, then by Hoeffding inequality, we have for each dimension $i$ and time $t\le T$:
\begin{equation}
P\left( |\eta\sum_{\tau=0}^{t-1}(1-\eta\mathcal{H})^{t-\tau-1}\xi_{\tau, i}|
>  \tlO(\eta^{\frac{1}{2}}\log{\frac{1}{\eta}})\right) \le e^{-\tlOmega(\log^2 \frac{1}{\eta})} \le \tlO(\eta^4)
\end{equation}

then by summing over dimension $d$ and taking union bound over all $t\le T$, we directly have:
\begin{equation}
P\left( \forall t\le T, \|\eta\sum_{\tau=0}^{t-1}(1-\eta\mathcal{H})^{t-\tau-1}\xi_{\tau}\|
>  \tlO(\eta^{\frac{1}{2}}\log \frac{1}{\eta})\right) \le \tlO(\eta^3).
\end{equation}
Combine this fact with Eq.(\ref{derivative_tilde}) and Eq.(\ref{dif_tilde}), we finish the proof.


\end{proof}

Next we need to prove that the two sequences of updates are always close.

\begin{lemma} \label{lem::saddle_and_maximum}
Under the assumptions of Theorem~\ref{thm:sgdmain_unconstraint}.
and let $\{w_t\}$ be the corresponding sequence generated by running SGD on function $f$.  
Also let $\tilde{f}$ and $\{\tilde{w}_t\}$ be defined as in Lemma \ref{lem::case_Gaussian}.
Then, for any initial point $w_0$ where $\|\nabla f(w_0)\| \le \tlO(\eta) < \epsilon$, and $\lambda_{\min}(\Hess f(w_0)) = -\gamma_0$. Given the choice of $T$ as in Eq.(\ref{choose_t}), 
with probability at least $1-\tlO(\eta^2)$, we have following holds simultaneously for all $t\le T$:
\begin{align}
\|w_t - \tilde{w}_t\| \le \tlO( \eta \log^2\frac{1}{\eta});
\quad\quad
\|\nabla f(w_t) - \nabla \tilde{f}(\tilde{w}_t)\| \le \tlO( \eta \log^2\frac{1}{\eta})
\end{align}
\end{lemma} 

\begin{proof}

First, we have update function of gradient by:
\begin{align}
\nabla f(w_t) = & \nabla f(w_{t-1}) + \int_{0}^{1}\mathcal{H}(w_{t-1} + t(w_t - w_{t-1})) \mathrm{d}t \cdot (w_t - w_{t-1}) \nonumber \\
= & \nabla f(w_{t-1}) + \mathcal{H}(w_{t-1}) (w_t - w_{t-1}) + \theta_{t-1}\label{derivative_recursive}
\end{align}
where the remainder:
\begin{equation}
\theta_{t-1} \equiv \int_{0}^{1}\left[\mathcal{H}(w_{t-1} + t(w_t - w_{t-1})) - \mathcal{H}(w_{t-1})\right] \mathrm{d}t \cdot (w_t - w_{t-1})
\end{equation}
Denote $\mathcal{H} = \mathcal{H}(w_0)$, and $\mathcal{H}'_{t-1} = \mathcal{H}(w_{t-1}) - \mathcal{H}(w_0)$.
By Hessian smoothness, we immediately have:
\begin{align} 
&\|\mathcal{H}'_{t-1}\| = \|\mathcal{H}(w_{t-1}) - \mathcal{H}(w_0)\| 
\le \rho \|w_{t-1} - w_0\| \le \rho (\|w_t - \tilde{w}_t\| + \|\tilde{w}_t - w_0\|) 
\label{H'_smooth} \\
&\|\theta_{t-1}\| \le \frac{\rho}{2} \|w_t - w_{t-1}\|^2 \label{theta_smooth}
\end{align}

Substitute the update equation of SGD (Eq.(\ref{SGD_update})) into Eq.(\ref{derivative_recursive}), we have:
\begin{align}
\nabla f(w_t) &= \nabla f(w_{t-1}) -\eta(\mathcal{H}+ \mathcal{H}'_{t-1}) ( \nabla f(w_{t-1}) + \xi_{t-1}) + \theta_{t-1} \nonumber \\
				&= (1-\eta\mathcal{H})\nabla f(w_{t-1}) - \eta \mathcal{H}\xi_{t-1} -\eta \mathcal{H}'_{t-1}
				(\nabla f(w_{t-1}) + \xi_{t-1})+ \theta_{t-1}  \label{derivative}
\end{align}

Let $\Delta_t = \nabla f(w_t) - \nabla \tilde{f}(\tilde{w}_t)$ denote the difference in gradient, then
from Eq.(\ref{derivative_tilde}), Eq.(\ref{derivative}), and Eq.(\ref{SGD_update}), we have:
\begin{align} \label{Delta_recursive}
&\Delta_t =  (1-\eta \mathcal{H}) \Delta_{t-1} -\eta \mathcal{H}'_{t-1} [\Delta_{t-1} + \nabla \tilde{f}(\tilde{w}_{t-1}) + \xi_{t-1}]  + \theta_{t-1}\\
&w_t - \tilde{w}_t =  -\eta \sum_{\tau = 0}^{t-1} \Delta_\tau \label{dif}
\end{align}

Let filtration $\mathfrak{F}_t = \sigma\{\xi_0, \cdots \xi_{t-1}\}$, and note $\sigma\{\Delta_0, \cdots, \Delta_t \} \subset \mathfrak{F}_t$, where $\sigma\{\cdot\}$ denotes the sigma field. Also, let event $\mathfrak{K}_t = \{\forall \tau \le t, ~\|\nabla\tilde{f}(\tilde{w}_\tau)\| \le \tlO(\eta^{\frac{1}{2}}\log \frac{1}{\eta}), ~ 
\|\tilde{w}_\tau - w_0\| \le \tlO(\eta^{\frac{1}{2}}\log \frac{1}{\eta})\}$,
and $\mathfrak{E}_t = \{\forall \tau \le t, ~\|\Delta_{\tau}\| \le \mu \eta\log^2\frac{1}{\eta}\}$, where $\mu$ is independent of $(\eta, \zeta)$, and will be specified later. Again, $\tilde{O}$ notation in this proof will never hide any dependence on $\mu$.
Clearly, we have 
$\mathfrak{K}_t \subset \mathfrak{K}_{t-1}$
($\mathfrak{E}_t \subset \mathfrak{E}_{t-1}$), 
thus $1_{\mathfrak{K}_t} \le 1_{\mathfrak{K}_{t-1}}$
($1_{\mathfrak{E}_t} \le 1_{\mathfrak{E}_{t-1}}$), where $1_\mathfrak{K}$ is the indicator function of event $\mathfrak{K}$.

We first need to carefully bounded all terms in Eq.(\ref{Delta_recursive}), 
conditioned on event $\mathfrak{K}_{t-1}\cap \mathfrak{E}_{t-1}$, by Eq.(\ref{H'_smooth}), Eq.(\ref{theta_smooth})), and Eq.(\ref{dif}),
with probability 1, for all $t\le T\le O(\log d/\gamma_0 \eta)$, we have: 
\begin{align}
\|(1-\eta \mathcal{H}) \Delta_{t-1}\| \le \tlO(\mu\eta\log^2\frac{1}{\eta}) 
&\quad \quad 
\|\eta  \mathcal{H}'_{t-1} (\Delta_{t-1} + \nabla \tilde{f}(\tilde{w}_{t-1}))\|
\le \tlO(\eta^2 \log^2 \frac{1}{\eta}) 
\nonumber \\
\|\eta\mathcal{H}'_{t-1}\xi_{t-1}\| \le \tlO(\eta^{1.5} \log \frac{1}{\eta})
&\quad \quad \|\theta_{t-1}\| \le \tlO(\eta^2) \label{order_of_term}
\end{align}

Since event $\mathfrak{K}_{t-1}\subset \mathfrak{F}_{t-1},  \mathfrak{E}_{t-1}\subset \mathfrak{F}_{t-1}$ thus independent of $\xi_{t-1}$, we also have:
\begin{align}
&\E [((1-\eta \mathcal{H}) \Delta_{t-1})^T\eta\mathcal{H}'_{t-1}\xi_{t-1} 1_{\mathfrak{K}_{t-1}\cap \mathfrak{E}_{t-1}} ~|~ \mathfrak{F}_{t-1}] \nonumber \\
=& 1_{\mathfrak{K}_{t-1}\cap \mathfrak{E}_{t-1}}((1-\eta \mathcal{H}) \Delta_{t-1})^T\eta\mathcal{H}'_{t-1}\E [\xi_{t-1} ~|~ \mathfrak{F}_{t-1}]
= 0
\end{align}

Therefore, from Eq.(\ref{Delta_recursive}) and Eq.(\ref{order_of_term}):
\begin{align}
&\E [\|\Delta_t\|^2_21_{\mathfrak{K}_{t-1}\cap \mathfrak{E}_{t-1}} ~|~ \mathfrak{F}_{t-1}] \nonumber \\
\le &\left[(1+\eta \gamma_0)^2\|\Delta_{t-1}\|^2 
+ (1+\eta \gamma_0) \|\Delta_{t-1}\| \tlO(\eta^2 \log^2 \frac{1}{\eta}) + \tlO(\eta^{3} \log^2 \frac{1}{\eta})\right]1_{\mathfrak{K}_{t-1}\cap \mathfrak{E}_{t-1}}\nonumber \\
\le& \left[(1+\eta \gamma_0)^2\|\Delta_{t-1}\|^2  + \tlO( \mu
\eta^{3}\log^4 \frac{1}{\eta})\right]1_{\mathfrak{K}_{t-1}\cap \mathfrak{E}_{t-1}} 
\end{align}

Define
\begin{align} \label{martingle_case3}
G_t = (1+\eta \gamma_0)^{-2t} [~\|\Delta_t\|^2 + \alpha 
\eta^{2}\log^4 \frac{1}{\eta}~]
\end{align}
Then, when $\eta_{\max}$ is small enough, we have:
\begin{align}
&\E [G_t1_{\mathfrak{K}_{t-1}\cap \mathfrak{E}_{t-1}} ~|~ \mathfrak{F}_{t-1}]  
= (1+\eta \gamma_0)^{-2t}\left[\E [\|\Delta_t\|^2_21_{\mathfrak{K}_{t-1}\cap \mathfrak{E}_{t-1}} ~|~ \mathfrak{F}_{t-1}]  + \alpha \eta^{2}\log^3 \frac{1}{\eta}\right]
1_{\mathfrak{K}_{t-1}\cap \mathfrak{E}_{t-1}} \nonumber \\
\le & (1+\eta \gamma_0)^{-2t}\left[(1+\eta \gamma_0)^2\|\Delta_{t-1}\|^2  +  \tlO(\mu\eta^{3}\log^4 \frac{1}{\eta})
+ \alpha \eta^{2}\log^4 \frac{1}{\eta}\right]
1_{\mathfrak{K}_{t-1}\cap \mathfrak{E}_{t-1}} \nonumber \\
\le & (1+\eta \gamma_0)^{-2t}\left[(1+\eta \gamma_0)^2\|\Delta_{t-1}\|^2 
+ (1+\eta \gamma_0)^2\alpha \eta^{2}\log^4 \frac{1}{\eta}
\right]1_{\mathfrak{K}_{t-1}\cap \mathfrak{E}_{t-1}} \nonumber \\
=  & G_{t-1}1_{\mathfrak{K}_{t-1}\cap \mathfrak{E}_{t-1}} \le G_{t-1}
1_{\mathfrak{K}_{t-2}\cap \mathfrak{E}_{t-2}}
\end{align}
Therefore, we have $\E [G_t1_{\mathfrak{K}_{t-1}\cap \mathfrak{E}_{t-1}} ~|~ \mathfrak{F}_{t-1}] \le G_{t-1}1_{\mathfrak{K}_{t-2}\cap \mathfrak{E}_{t-2}}$ which means 
$G_t1_{\mathfrak{K}_{t-1}\cap \mathfrak{E}_{t-1}}$ is a supermartingale.

On the other hand, we have:
\begin{align}
\Delta_t =  (1-\eta H) \Delta_{t-1} -\eta  \mathcal{H}'_{t-1} (\Delta_{t-1} + \nabla \tilde{f}(\tilde{w}_{t-1})) -\eta \mathcal{H}'_{t-1}\xi_{t-1}  + \theta_{t-1}
\end{align}
Once conditional on filtration $\mathfrak{F}_{t-1}$, the first two terms are deterministic, and only the third and fourth term are random.
Therefore, we know, with probability 1:
\begin{equation}
|~\|\Delta_t\|^2_2 - \E[\|\Delta_t\|^2_2|\mathfrak{F}_{t-1}]~|1_{\mathfrak{K}_{t-1}\cap \mathfrak{E}_{t-1}} \le \tlO(\mu\eta^{2.5} \log^3 \frac{1}{\eta}) 
\end{equation}
Where the main contribution comes from the product of the first term and third term.
Then, with probability 1, we have:
\begin{align}
&| G_t1_{\mathfrak{K}_{t-1}\cap \mathfrak{E}_{t-1}} - \E[G_t1_{\mathfrak{K}_{t-1}\cap \mathfrak{E}_{t-1}}~|~\mathfrak{F}_{t-1}] | \nonumber \\
= &(1+2\eta \gamma_0)^{-2t}\cdot|~\|\Delta_t\|^2_2 - \E[\|\Delta_t\|^2_2|\mathfrak{F}_{t-1}]~|\cdot1_{\mathfrak{K}_{t-1}\cap \mathfrak{E}_{t-1}} 
 \le \tlO(\mu\eta^{2.5} \log^3 \frac{1}{\eta}) = c_{t-1}
\end{align}
By Azuma-Hoeffding inequality, with probability less than $\tlO(\eta^3)$, 
for $t\le T\le O(\log d/\gamma_0 \eta)$:
\begin{equation}
G_t1_{\mathfrak{K}_{t-1}\cap \mathfrak{E}_{t-1}} - G_0\cdot1 > \tlO(1)\sqrt{\sum_{\tau=0}^{t-1}{c^2_\tau}}\log (\frac{1}{\eta}) = \tlO(\mu\eta^{2}\log^4 \frac{1}{\eta})
\end{equation}
This means there exist some $\tilde{C} = \tlO(1)$ so that:
\begin{equation}
	P\left(G_t1_{\mathfrak{K}_{t-1}\cap \mathfrak{E}_{t-1}}  \ge \tilde{C}\mu\eta^{2}\log^4 \frac{1}{\eta}\right) \le \tlO(\eta^3)
\end{equation}
By choosing $\mu>\tilde{C}$, this is equivalent to:
\begin{equation}
	P\left(\mathfrak{K}_{t-1}\cap \mathfrak{E}_{t-1} \cap \left\{\|\Delta_t\|^2 \ge \mu^2 \eta^{2}\log^4 \frac{1}{\eta}\right\}\right) \le \tlO(\eta^3)
\end{equation}
Therefore, combined with Lemma \ref{lem::case_Gaussian}, we have:
\begin{align}
&P\left( \mathfrak{E}_{t-1} \cap \left\{\|\Delta_t\| \ge \mu \eta \log^2 \frac{1}{\eta}\right\}\right)  \nonumber \\
= &P\left(\mathfrak{K}_{t-1}\cap \mathfrak{E}_{t-1} \cap \left\{\|\Delta_t\| \ge \mu\eta \log^2 \frac{1}{\eta}\right\}\right) +
P\left(\overline{\mathfrak{K}}_{t-1}\cap \mathfrak{E}_{t-1} \cap \left\{\|\Delta_t\| \ge \mu\eta \log^2 \frac{1}{\eta}\right\}\right) \nonumber\\
\le &\tlO(\eta^3) + P(\overline{\mathfrak{K}}_{t-1}) \le \tlO(\eta^3)
\end{align}
Finally, we know:
\begin{align}
P(\overline{ \mathfrak{E}}_{t}) = P\left( \mathfrak{E}_{t-1} \cap \left\{\|\Delta_t\| \ge \mu\eta \log^2 \frac{1}{\eta}\right\}\right) + P(\overline{ \mathfrak{E}}_{t-1}) 
\le \tlO(\eta^3) + P(\overline{ \mathfrak{E}}_{t-1}) 
\end{align} 
Because $P(\overline{ \mathfrak{E}}_{0}) =0$, and $T\le \tlO(\frac{1}{\eta})$, we have 
$P(\overline{ \mathfrak{E}}_{T}) \le \tlO(\eta^2)$.
Due to Eq.(\ref{dif}), we have $\|w_t - \tilde{w}_t\| \le  \eta \sum_{\tau = 0}^{t-1} \|\Delta_\tau\|$, then by the definition of $\mathfrak{E}_{T}$, we finish the proof.

\end{proof}

Using the two lemmas above we are ready to prove Lemma~\ref{thm::case2}

\begin{proof}[Proof of Lemma \ref{thm::case2}]
Let $\tilde{f}$ and $\{\tilde{w}_t\}$ be defined as in Lemma \ref{lem::case_Gaussian}.
and also let $\lambda_{\min}(\mathcal{H} (w_0)) =-\gamma_0$. Since $\mathcal{H}(w)$ is 
$\rho$-Lipschitz, for any $w, w_0$, we have:
\begin{equation}
f(w) \le  f(w_0) + \nabla f(w_0)^T (w-w_0) + \frac{1}{2}(w-w_0)^T \mathcal{H}(w_0) (w-w_0)
+ \frac{\rho}{6} \|w-w_0\|^3
\end{equation}

Denote $\tilde{\delta} = \tilde{w}_T - w_0$ and $\delta = w_T - \tilde{w}_T$, we have: 
\begin{align}
	f(w_T) -  f(w_0) \le& \left[\nabla f(w_0)^T (w_T-w_0) + \frac{1}{2}(w_T-w_0)^T \mathcal{H}(w_0) (w_T-w_0)+ \frac{\rho}{6} \|w_T-w_0\|^3\right] \nonumber \\
	=&\left[\nabla f(w_0)^T (\tilde{\delta} + \delta) + \frac{1}{2}(\tilde{\delta} + \delta)^T \mathcal{H}(\tilde{\delta} + \delta) + \frac{\rho}{6}\|\tilde{\delta} + \delta\|^3\right]\nonumber \\
	=&\left[\nabla f(w_0)^T \tilde{\delta} + \frac{1}{2}\tilde{\delta}^T \mathcal{H}\tilde{\delta}\right] 
	+ \left[\nabla f(w_0)^T \delta + 
	\tilde{\delta}^T \mathcal{H}\delta + 
	\frac{1}{2}\delta^T \mathcal{H}\delta + 
	\frac{\rho}{6}\|\tilde{\delta} + \delta\|^3\right]
\end{align}
Where $\mathcal{H} = \mathcal{H}(w_0)$. Denote $\tilde{\Lambda} = \nabla f(w_0)^T \tilde{\delta} + \frac{1}{2}\tilde{\delta}^T \mathcal{H}\tilde{\delta}$ be the first term, and $\Lambda = \nabla f(w_0)^T \delta + 
\tilde{\delta}^T \mathcal{H}\delta + 
\frac{1}{2}\delta^T \mathcal{H}\delta + 
\frac{\rho}{6}\|\tilde{\delta} + \delta\|^3$ be the second term.
We have $f(w_T) -  f(w_0) \le \tilde{\Lambda} + \Lambda$.

Let $\mathfrak{E}_t = \{\forall \tau \le t, 
\|\tilde{w}_\tau - w_0\| \le \tlO(\eta^{\frac{1}{2}}\log \frac{1}{\eta}), 
~\|w_t - \tilde{w}_t\| \le \tlO(\eta\log^2 \frac{1}{\eta})\}$, by the result of Lemma \ref{lem::case_Gaussian} and Lemma \ref{lem::saddle_and_maximum}, we know $P(\mathfrak{E}_T)\ge 1-\tlO(\eta^2)$. Then, clearly, we have:

\begin{align}\label{Lambda_decomp}
	\E f(w_T) -  f(w_0) =& \E [f(w_T) -  f(w_0)]1_{\mathfrak{E}_T} + \E [f(w_T) -  f(w_0)]1_{\overline{\mathfrak{E}}_T} \nonumber \\
	\le & \E \tilde{\Lambda}1_{\mathfrak{E}_T} + \E \Lambda1_{\mathfrak{E}_T} + \E [f(w_T) -  f(w_0)]1_{\overline{\mathfrak{E}}_T} \nonumber \\
	= & \E \tilde{\Lambda} + \E \Lambda1_{\mathfrak{E}_T} + \E [f(w_T) -  f(w_0)]1_{\overline{\mathfrak{E}}_T} - \E \tilde{\Lambda}1_{\overline{\mathfrak{E}}_T}
\end{align}
We will carefully caculate $\E\tilde{\Lambda}$ term first, and then bound remaining term as ``perturbation'' to first term.

Let $\lambda_1, \cdots, \lambda_d$ be the eigenvalues of $\mathcal{H}$. By the result of lemma \ref{lem::case_Gaussian} and simple linear algebra, we have:
\begin{align} 
	\E\tilde{\Lambda} &= - \frac{\eta}{2}\sum_{i=1}^d\sum_{\tau=0}^{2T-1}(1-\eta\lambda_i)^\tau |\nabla_i  f(w_0)|^2
	+ \frac{1}{2}\sum_{i=1}^d \lambda_i \sum_{\tau=0}^{T-1}(1-\eta\lambda_i)^{2\tau}\eta^2\sigma^2
	\nonumber \\
	&\le \frac{1}{2}\sum_{i=1}^d \lambda_i \sum_{\tau=0}^{T-1}(1-\eta\lambda_i)^{2\tau}\eta^2\sigma^2
	\nonumber\\
	&\le \frac{\eta^2\sigma^2}{2}\left[\frac{d-1}{\eta} - \gamma_0\sum_{\tau=0}^{T-1}(1+\eta \gamma_0)^{2\tau}\right]
	\le -\frac{\eta\sigma^2}{2}\label{Lambda_tilde}
\end{align}
The last inequality is directly implied by the choice of $T$ as in Eq.(\ref{choose_t}).
Also, by Eq.(\ref{choose_t}), we also immediately have that 
$T = O(\log d/\gamma_0\eta) \le O(\log d/\gamma\eta)$. 
Therefore, by choose $T_{max}=O(\log d/\gamma\eta)$ with large enough constant, we have
$T \le T_{max} = O(\log d/\gamma\eta)$.

For bounding the second term, by definition of $\mathfrak{E}_t$, we have:
\begin{align}\label{Lambda_E}
	\E\Lambda 1_{\mathfrak{E}_T}=\E\left[\nabla f(w_0)^T \delta + 
	\tilde{\delta}^T \mathcal{H}\delta + 
	\frac{1}{2}\delta^T \mathcal{H}\delta + 
	\frac{\rho}{6}\|\tilde{\delta} + \delta\|^3\right]1_{\mathfrak{E}_T}
	\le \tlO(\eta^{1.5}\log^3 \frac{1}{\eta})
\end{align}

On the other hand,
since noise is bounded as $\|\xi\| \le \tlO(1)$, from the results of Lemma \ref{lem::case_Gaussian}, it's easy to show $\|\tilde{w} - w_0\| = \|\tilde{\delta}\| \le \tlO(1)$ is also bounded with probability 1. Recall the assumption that function $f$ is also bounded, then we have:
\begin{align}\label{Lambda_E_bar}
	&\E [f(w_T) -  f(w_0)]1_{\overline{\mathfrak{E}}_T} - \E \tilde{\Lambda}1_{\overline{\mathfrak{E}}_T} \nonumber \\
	= &\E [f(w_T) -  f(w_0)]1_{\overline{\mathfrak{E}}_T} - \E\left[\nabla f(w_0)^T \tilde{\delta} + \frac{1}{2}\tilde{\delta}^T \mathcal{H}\tilde{\delta}\right]1_{\overline{\mathfrak{E}}_T} 
	\le \tlO(1)P(\overline{\mathfrak{E}}_T) \le \tlO(\eta^2)
\end{align}

Finally, substitute Eq.(\ref{Lambda_tilde}), Eq.(\ref{Lambda_E}) and Eq.(\ref{Lambda_E_bar})
into Eq.(\ref{Lambda_decomp}), we finish the proof.
\end{proof}

Finally, we combine three cases to prove the main theorem.

\begin{proof} [Proof of Theorem \ref{thm:sgdmain_unconstraint}]
Let's set $\mathcal{L}_1 = \{w ~|~ \|\nabla f(w)\| \ge \sqrt{2\eta\sigma^2\beta d}\}$, $\mathcal{L}_2 = \{w ~|~ \|\nabla f(w)\| \le \sqrt{2\eta\sigma^2\beta d}$ and $\lambda_{\min} (\mathcal{H}(w) ) \le -\gamma\}$, and $\mathcal{L}_3 = \mathcal{L}^c_1 \cup \mathcal{L}^c_2$.
By choosing small enough $\eta_{\max}$, we could make $\sqrt{2\eta\sigma^2\beta d} < \min\{\epsilon, \alpha \delta\}$. Under this choice, we know from Definition~\ref{def:robustcondition} of ($\alpha, \gamma, \epsilon, \delta$)-\name that $\mathcal{L}_3$ is the locally $\alpha$-strongly convex region which is $\tlO(\sqrt{\eta})$-close to some local minimum.

We shall first prove that within $\tlO(\frac{1}{\eta^2}\log \frac{1}{\zeta})$ steps with probability at least $1-\zeta/2$ one of $w_t$ is in $\mathcal{L}_3$. Then by Lemma~\ref{thm::case3} we know with probability at most $\zeta/2$ there exists a $w_t$ that is in $\mathcal{L}_3$ but the last point is not. By union bound we will get the main result.

To prove within $\tlO(\frac{1}{\eta^2}\log \frac{1}{\zeta})$ steps with probability at least $1-\zeta/2$ one of $w_t$ is in $\mathcal{L}_3$, we first show starting from any point, in $\tlO(\frac{1}{\eta^2})$ steps with probability at least $1/2$ one of $w_t$ is in $\mathcal{L}_3$. Then we can repeat this $\log 1/\zeta$ times to get the high probability result.

Define stochastic process $\{\tau_i\}$ s.t. $\tau_0 = 0$, and 
\begin{equation}
\tau_{i+1} = 
\begin{cases}
	\tau_i + 1 &\mbox{if~} w_{\tau_i} \in \mathcal{L}_1 \cup \mathcal{L}_3\\
	\tau_i + T(w_{\tau_i}) &\mbox{if~} w_{\tau_i} \in \mathcal{L}_2
\end{cases}
\end{equation}
Where $T(w_{\tau_i})$ is defined by Eq.(\ref{choose_t}) with $\gamma_0= \lambda_{\min}(\mathcal{H}(w_{\tau_i}))$and we know $T \le T_{max} = \tlO(\frac{1}{\eta})$.

By Lemma \ref{thm::case1} and Lemma \ref{thm::case2}, we know:
\begin{align}
&\E [f(w_{\tau_{i+1}}) - f(w_{\tau_i})|w_{\tau_i} \in \mathcal{L}_1, \mathfrak{F}_{\tau_i - 1} ] 
= \E [f(w_{\tau_{i+1}}) - f(w_{\tau_i})|w_{\tau_i} \in \mathcal{L}_1] \le -\tlO(\eta^2) \\
&\E [f(w_{\tau_{i+1}}) - f(w_{\tau_i})|w_{\tau_i} \in \mathcal{L}_2, \mathfrak{F}_{\tau_i - 1} ] 
= \E [f(w_{\tau_{i+1}}) - f(w_{\tau_i})|w_{\tau_i} \in \mathcal{L}_2] \le -\tlO(\eta)
\end{align}
Therefore, combine above equation, we have:
\begin{equation}
\E [f(w_{\tau_{i+1}}) - f(w_{\tau_i})|w_{\tau_i} \not\in \mathcal{L}_3, \mathfrak{F}_{\tau_i - 1} ] 
= \E [f(w_{\tau_{i+1}}) - f(w_{\tau_i})|w_{\tau_i} \not\in \mathcal{L}_3] \le -(\tau_{i+1}-\tau_i)\tlO(\eta^2)
\end{equation}

Define event $\mathfrak{E}_i = \{\exists j \le i, ~ w_{\tau_j} \in \mathcal{L}_3\}$, clearly
$\mathfrak{E}_i \subset \mathfrak{E}_{i+1}$, thus $P(\mathfrak{E}_i) \le P(\mathfrak{E}_{i+1}) $. Finally,
consider $f(w_{\tau_{i+1}})1_{\mathfrak{E}_{i}}$, we have:
\begin{align}
	\E f(w_{\tau_{i+1}})1_{\mathfrak{E}_{i}} - \E f(w_{\tau_i})1_{\mathfrak{E}_{i-1}}
	&\le B \cdot P(\mathfrak{E}_{i} - \mathfrak{E}_{i-1}) + \E[f(w_{\tau_{i+1}}) - f(w_{\tau_{i}})|\overline{\mathfrak{E}_{i}}]\cdot P(\overline{\mathfrak{E}_{i}}) \nonumber \\
	&\le B \cdot P(\mathfrak{E}_{i}- \mathfrak{E}_{i-1}) -(\tau_{i+1}-\tau_i)\tlO(\eta^2)P(\overline{\mathfrak{E}_{i}})
\end{align}
Therefore, by summing up over $i$, we have: 
\begin{equation}
	\E f(w_{\tau_{i}})1_{\mathfrak{E}_{i}} -  f(w_{0})
	\le BP(\mathfrak{E}_{i}) -\tau_{i}\tlO(\eta^2)P(\overline{\mathfrak{E}_{i}})
	\le B-\tau_{i}\tlO(\eta^2)P(\overline{\mathfrak{E}_{i}})
\end{equation}
Since $|f(w_{\tau_{i}})1_{\mathfrak{E}_{i}}| < B$ is bounded, as $\tau_i$ grows to as large as $ \frac{6B}{\eta^2}$, we must have $P(\overline{\mathfrak{E}_{i}}) < \frac{1}{2}$.
That is, after $\tlO(\frac{1}{\eta^2})$ steps, with at least probability $1/2$, $\{w_t\}$ have at least enter $\mathcal{L}_3$ once. Since this argument holds for any starting point, we can repeat this $\log 1/\zeta$ times and we know after $\tlO(\frac{1}{\eta^2}\log 1/\zeta)$ steps, with probability at least $1-\zeta/2$, $\{w_t\}$ have at least enter $\mathcal{L}_3$ once.

Combining with Lemma~\ref{thm::case3}, and by union bound we know
 after $\tlO(\frac{1}{\eta^2}\log 1/\zeta)$ steps, with probability at least $1-\zeta$, $w_t$
will be in the $\tlO(\sqrt{\eta \log \frac{1}{\eta\zeta}})$ neigborhood of some local minimum.
\end{proof}

\clearpage

\section{Detailed Analysis for Section~\ref{sec:sgd} in Constrained Case}
\label{sec:constrained}

So far, we have been discussed all about unconstrained problem. In this section we extend our result to equality constraint problems under some mild conditions.

Consider the equality constrained optimization problem:
\begin{align}\label{eq_constraint_problem}
&\min_w \quad \quad  f(w) \\
&\text{s.t.} \quad \quad c_i(w) = 0, \quad \quad i=1, \cdots, m \nonumber
\end{align}
Define the feasible set as the set of points that satisfy all the constraints $\mathcal{W} = \{w~|~c_i(w) = 0; ~ i=1, \cdots, m \}$.

In this case, the algorithm we are running is Projected Noisy Gradient Descent.
Let function $\Pi_{\mathcal{W}}(v)$ to be the projection to the feasible set， where the projection is 
defined as the global solution of $\min_{w \in \mathcal{W}} \|v-w\|^2$.

With same argument as in the unconstrained case, we could slightly simplify and convert it to standard projected stochastic gradient descent (PSGD) with update equation:
\begin{align} \label{PSGD_update}
	&v_t = w_{t-1} - \eta \nabla f(w_{t-1}) + \xi_{t-1} \\
	&w_t = \Pi_{\mathcal{W}}(v_t)
\end{align}
As in unconstrained case, we are interested in noise $\xi$ is i.i.d satisfying $\E \xi = 0$, 
$ \E \xi\xi^T = \sigma^2I$ and $\|\xi\| \le Q$ almost surely.
Our proof can be easily extended to Algorithm \ref{algo:psgdwn} with $\frac{1}{d} I \preceq \E\xi\xi^T\preceq (Q+\frac{1}{d})I$. 
In this section we first introduce basic tools for handling constrained optimization problems (most these materials can be found in~\cite{wright1999numerical}), then we prove some technical lemmas that are useful for dealing with the projection step in PSGD, finally we point out how to modify the previous analysis.

\subsection{Preliminaries}

Often for constrained optimization problems we want the constraints to satisfy some regularity conditions. LICQ (linear independent constraint quantification) is a common assumption in this context.

\begin{definition}[LICQ]
In equality-constraint problem Eq.(\ref{eq_constraint_problem}), given a point $w$, 
we say that the linear independence constraint qualification (LICQ) holds if the set of constraint gradients $\{\nabla c_i(x), i=1, \cdots, m\}$ is linearly independent.
\end{definition}

In constrained optimization, we can locally transform it to an unconstrained problem by introducing Lagrangian multipliers. The Langrangian $\mathcal{L}$ can be written as
\begin{equation}
\mathcal{L}(w, \lambda) =  f(w) - \sum_{i=1}^m \lambda_i c_i(w)
\end{equation}

Then, if LICQ holds for all $w \in \mathcal{W}$, we can properly define function $\lambda^*(\cdot)$ to be:
\begin{equation}
\lambda^*(w) = \arg\min_{\lambda} \|\nabla f(w) - \sum_{i=1}^m \lambda_i \nabla c_i(w)\|
= \arg\min_{\lambda} \|\nabla_w \mathcal{L}(w, \lambda)\|
\end{equation}
where $\lambda^*(\cdot)$ can be calculated analytically:
let matrix $C(w) = (\nabla c_1(w), \cdots, \nabla c_m(w))$, then we have:
\begin{equation} \label{Lambda_star}
\lambda^*(w)  = C(w)^{\dagger}\nabla f(w) = (C(w)^TC(w))^{-1}C(w)^T\nabla f(w)
\end{equation}
where $(\cdot)^{\dagger}$ is Moore-Penrose pseudo-inverse.

In our setting we need a stronger regularity condition which we call robust LICQ (RLICQ).

\begin{definition}[\nameCQ]
In equality-constraint problem Eq.(\ref{eq_constraint_problem}), given a point $w$, 
we say that $\alpha_c$-robust linear independence constraint qualification (\nameCQ) holds if the minimum singular value of matrix $C(w) = (\nabla c_1(w), \cdots, \nabla c_m(w))$ is greater or equal to $\alpha_c$, that is $\sigma_{\min} (C(w)) \ge \alpha_c$.
\end{definition}

\begin{remark}
Given a point $w\in \mathcal{W}$, \nameCQ implies LICQ. 
While LICQ holds for all $w \in \mathcal{W}$ is a necessary condition for
$\lambda^*(w)$ to be well-defined;
it's easy to check that \nameCQ holds for all $w \in \mathcal{W}$ is a necessary condition for
$\lambda^*(w)$ to be bounded. Later, we will also see \nameCQ combined with the smoothness of $\{c_i(w)\}_{i=1}^m$ guarantee the curvature of constraint manifold to be bounded everywhere.
\end{remark}

Note that we require this condition in order to provide a quantitative bound, without this assumption there can be cases that are exponentially close to a function that does not satisfy LICQ.

We can also write down the first-order and second-order partial derivative of Lagrangian $\mathcal{L}$
at point $(w, \lambda^*(w))$:
\begin{align}
&\chi(w) = \nabla_w \mathcal{L}(w, \lambda) |_{(w, \lambda^*(w))} =\nabla f(w) - \sum_{i=1}^m \lambda^*_i(w) \nabla c_i(w)  \label{Lagrangian_1}\\
&\mathfrak{M}(w) = \nabla^2_{ww} \mathcal{L}(w, \lambda) |_{(w, \lambda^*(w))} = \Hess f(w) - \sum_{i=1}^m \lambda^*_i(w) \nabla^2 c_i(w) \label{Lagrangian_2}
\end{align}

\begin{definition}[Tangent Space and Normal Space]
Given a feasible point $w \in \mathcal{W}$, define its corresponding Tangent Space to be $\mathcal{T}(w) = \{v ~|~\nabla c_i(w)^T v = 0; ~ i=1, \cdots, m \}$, and Normal Space to be $\mathcal{T}^c(w) = \text{span}\{\nabla c_1(w) \cdots, \nabla c_m(w) \}$
\end{definition}
If $w \in \mathcal{R}^d$, and we have $m$ constraint satisfying \nameCQ, the tangent space would be
a linear subspace with dimension $d-m$; and the normal space would be a linear subspace with dimension $m$.
We also know immediately that $\chi(w)$ defined in Eq.(\ref{Lagrangian_1}) has another interpretation: it's the component of gradient $\nabla f(w)$ in tangent space.

Also, it's easy to see the normal space $\mathcal{T}^c(w)$ is the orthogonal complement of 
$\mathcal{T}$. We can also define the projection matrix of any vector onto tangent space (or normal space) to be $P_{\mathcal{T}(w)}$ (or $P_{\mathcal{T}^c(w)}$). Then, clearly, both 
$P_{\mathcal{T}(w)}$ and $P_{\mathcal{T}^c(w)}$ are orthoprojector, thus symmetric.
Also by Pythagorean theorem, we have:
\begin{equation}
\|v\|^2 = \|P_{\mathcal{T}(w)} v\|^2 + \|P_{\mathcal{T}^c(w)}v\|^2, \quad \quad \forall
v\in \mathbb{R}^d
\end{equation}

\paragraph{Taylor Expansion} Let $w, w_0 \in \mathcal{W}$, and fix $\lambda^* = \lambda^*(w_0)$ independent of $w$, assume 
$\nabla^2_{ww} \mathcal{L}(w, \lambda^*)$ is $\rho_L$-Lipschitz,
that is $\|\nabla^2_{ww} \mathcal{L}(w_1, \lambda^*) - \nabla^2_{ww} \mathcal{L}(w_2, \lambda^*) \| \le \rho_L \|w_1 - w_2\|$
By Taylor expansion, we have:
\begin{align}
\mathcal{L}(w, \lambda^*) \le &\mathcal{L}(w_0, \lambda^*) + \nabla_w \mathcal{L}(w_0, \lambda^*)^T (w-w_0)  \nonumber \\
& + \frac{1}{2}(w-w_0)^T \nabla^2_{ww} \mathcal{L}(w_0, \lambda^*) (w-w_0) + \frac{\rho_L}{6}\|w-w_0\|^3
\end{align}
Since $w, w_0$ are feasible, we know:
$\mathcal{L}(w, \lambda^*) = f(w)$ and $\mathcal{L}(w_0, \lambda^*) = f(w_0)$, 
this gives:
\begin{align} \label{Taylor_eq_constraint}
f(w)\le f(w_0) + \chi(w_0)^T (w-w_0) + \frac{1}{2}(w-w_0)^T \mathfrak{M}(w_0)(w-w_0) + \frac{\rho_L}{6}\|w-w_0\|^3
\end{align}


\paragraph{Derivative of $\chi(w)$} By taking derative of $\chi(w)$ again, we know the change of this tangent gradient can be characterized by:
\begin{align}
	\nabla \chi(w) = \mathcal{H} - \sum_{i=1}^m \lambda^*_i(w) \nabla^2 c_i(w) 
	-\sum_{i=1}^m  \nabla c_i(w) \nabla \lambda^*_i(w)^T
\end{align}
Denote 
\begin{equation}
\mathfrak{N}(w) =-\sum_{i=1}^m  \nabla c_i(w) \nabla \lambda^*_i(w)^T
\end{equation}
We immediately know that $\nabla \chi(w) = \mathfrak{M}(w) + \mathfrak{N}(w)$.

\begin{remark}\label{N_constraint}
The additional term $\mathfrak{N}(w)$ is not necessary to be even symmetric in general. This is due to the fact that $\chi(w)$ may not be the gradient of any scalar function. However, 
$\mathfrak{N}(w)$ has an important property that is: for any vector $v \in \mathbb{R}^d$, $\mathfrak{N}(w)v \in \mathcal{T}^c(w)$.
\end{remark}

Finally, for completeness, we state here the first/second-order necessary (or sufficient) conditions for optimality. Please refer to \cite{wright1999numerical} for the proof of those theorems.

\begin{theorem}[First-Order Necessary Conditions] \label{thm::first_necessary}
In equality constraint problem Eq.(\ref{eq_constraint_problem}), suppose that $w^\dagger$ is a local solution,
and that the functions $f$ and $c_i$ are continuously differentiable, and that the LICQ holds at $w^\dagger$.
Then there is a Lagrange multiplier vector $\lambda^\dagger$, such that:
\begin{align}
\nabla_w \mathcal{L}(w^\dagger, \lambda^\dagger) &= 0 \\
c_i(w^\dagger) &= 0, \quad \quad \text{for~} i=1, \cdots, m
\end{align}
These conditions are also usually referred as Karush-Kuhn-Tucker (KKT) conditions.
\end{theorem}

\begin{theorem}[Second-Order Necessary Conditions] \label{thm::second_necessary}
In equality constraint problem Eq.(\ref{eq_constraint_problem}), suppose that $w^\dagger$ is a local solution,
and that the LICQ holds at $w^\dagger$.
Let $\lambda^\dagger$ Lagrange multiplier vector for which the KKT conditions are satisfied. Then:
\begin{align}
v^T\nabla^2_{xx} \mathcal{L}(w^\dagger, \lambda^\dagger)v \ge 0 \quad \quad 
\text{for all~} v\in \mathcal{T}(w^\dagger)
\end{align}
\end{theorem}

\begin{theorem}[Second-Order Sufficient Conditions] \label{thm::second_sufficient}
In equality constraint problem Eq.(\ref{eq_constraint_problem}), suppose that for some feasible point $w^\dagger \in \mathbb{R}^d$, and there's  Lagrange multiplier vector
$\lambda^\dagger$  for which the KKT conditions are satisfied. Suppose also that:
\begin{align}
v^T\nabla^2_{xx} \mathcal{L}(w^\dagger, \lambda^\dagger)v > 0 \quad \quad 
\text{for all~} v\in \mathcal{T}(w^\dagger), v\neq 0
\end{align}
Then $w^\dagger$ is a strict local solution.
\end{theorem}

\begin{remark}
By definition Eq.(\ref{Lambda_star}), we know immediately $\lambda^*(w^\dagger)$ is one of valid Lagrange multipliers $\lambda^\dagger$ for which the KKT conditions are satisfied. This means $\chi(w^\dagger) = 
\nabla_w \mathcal{L}(w^\dagger, \lambda^\dagger)$ and $\mathfrak{M}(w^\dagger) = \mathcal{L}(w^\dagger, \lambda^\dagger)$.
\end{remark}

Therefore, Theorem \ref{thm::first_necessary}, \ref{thm::second_necessary}, \ref{thm::second_sufficient} gives strong implication that $\chi(w)$ and $\mathfrak{M}(w)$ are the right thing to look at, which are in some sense equivalent to $\nabla f(w)$ and $\Hess f(w)$ in unconstrained case.

\subsection{Geometrical Lemmas Regarding Constraint Manifold}

Since in equality constraint problem, at each step of PSGD, we are effectively considering the local manifold around feasible point $w_{t-1}$.
In this section, we provide some technical lemmas relating to the geometry of constraint manifold in preparsion for the proof of main theorem in equality constraint case.

We first show if two points are close, then the projection in the normal space is much smaller than the projection in the tangent space.

\begin{lemma} \label{lem::normal_by_tangent}
Suppose the constraints $\{c_i\}_{i=1}^m$ are $\beta_i$-smooth, 
and \nameCQ holds for all $w\in \mathcal{W}$.
Then, let $\sum_{i=1}^m\frac{\beta_i^2}{ \alpha^2_c} = \frac{1}{R^2}$,
for any $w, w_0 \in \mathcal{W}$, let $\mathcal{T}_0 = \mathcal{T}(w_0)$, then
\begin{equation}
\|P_{\mathcal{T}^c_0} (w-w_0)\| \le \frac{1}{2R} \|w-w_0\|^2
\end{equation}
Furthermore, if $\|w-w_0\| < R$ holds, 
we additionally have:
\begin{equation}
\|P_{\mathcal{T}^c_0} (w-w_0)\| \le \frac{\|P_{\mathcal{T}_0} (w-w_0)\|^2}{R}
\end{equation}
\end{lemma}

\begin{proof}
First, since for any vector $\hat{v} \in \mathcal{T}_0$, we have 
$\|C(w_0)^T \hat{v}\| = 0$, then
by simple linear algebra, it's easy to show:
\begin{align} \label{C1_constraint}
\|C(w_0)^T (w-w_0)\|^2  =&  \|C(w_0)^T P_{\mathcal{T}^c_0} (w-w_0)\|^2 
\ge \sigma^2_{\min} \|P_{\mathcal{T}^c_0} (w-w_0)\|^2 \nonumber \\
\ge & \alpha_c^2 \|P_{\mathcal{T}^c_0} (w-w_0)\|^2
\end{align}
On the other hand, by $\beta_i$-smooth, we have:
\begin{align}
|c_i(w) - c_i(w_0) - \nabla c_i(w_0) ^T (w-w_0)| \le \frac{\beta_i}{2} \|w-w_0\|^2
\end{align}
Since $w, w_0$ are feasible points, we have $c_i(w) = c_i(w_0) = 0$, 
which gives:
\begin{equation} \label{C2_constraint}
\|C(w_0)^T (w-w_0)\|^2 = \sum_{i=1}^m (\nabla c_i(w_0) ^T (w-w_0))^2
\le \sum_{i=1}^m \frac{\beta_i^2}{4} \|w-w_0\|^4
\end{equation}
Combining Eq.(\ref{C1_constraint}) and Eq.(\ref{C2_constraint}), and the definition of $R$, 
we have:
\begin{equation}
\|P_{\mathcal{T}^c_0} (w-w_0)\|^2 \le \frac{1}{4R^2} \|w-w_0\|^4
= \frac{1}{4R^2}(\|P_{\mathcal{T}^c_0} (w-w_0)\|^2 + \|P_{\mathcal{T}_0} (w-w_0)\|^2)^2
\end{equation}
Solving this second-order inequality gives two solution
\begin{equation}
	\|P_{\mathcal{T}^c_0} (w-w_0)\| \le \frac{\|P_{\mathcal{T}_0} (w-w_0)\|^2}{R}
	\quad \text{or} \quad \|P_{\mathcal{T}^c_0} (w-w_0)\| \ge R
\end{equation}
By assumption, we know $\|w-w_0\| < R$ (so the second case cannot be true), which finishes the proof.
\end{proof}

Here, we see the $\sqrt{\sum_{i=1}^m\frac{\beta_i^2}{ \alpha^2_c}} = \frac{1}{R}$ serves as a 
upper bound of the curvatures on the constraint manifold, and equivalently, $R$ serves as a lower bound of the radius of curvature. \nameCQ and smoothness guarantee that the curvature is bounded.

Next we show the normal/tangent space of nearby points are close.

\begin{lemma} \label{lem::normal}
Suppose the constraints $\{c_i\}_{i=1}^m$ are $\beta_i$-smooth, 
and \nameCQ holds for all $w\in \mathcal{W}$.
Let $\sum_{i=1}^m\frac{\beta_i^2}{ \alpha^2_c} = \frac{1}{R^2}$,
for any $w, w_0 \in \mathcal{W}$, let $\mathcal{T}_0 = \mathcal{T}(w_0)$.
Then for all $\hat{v} \in \mathcal{T}(w)$ so that $\|\hat{v}\| = 1$, we have 
\begin{equation}
\|P_{\mathcal{T}^c_0} \cdot \hat{v}\|  \le \frac{\|w-w_0\|}{R}
\end{equation}
\end{lemma}

\begin{proof}
With similar calculation as Eq.(\ref{C1_constraint}), we immediately have:
\begin{align} \label{C3_constraint}
\|P_{\mathcal{T}^c_0} \cdot \hat{v}\|^2
\le \frac{\|C(w_0)^T \hat{v}\|^2 }{\sigma^2_{\min}(C(w))}
\le \frac{\|C(w_0)^T \hat{v}\|^2 }{\alpha_c^2}
\end{align}
Since $\hat{v} \in \mathcal{T}(w)$ , we have $C(w)^T \hat{v} = 0$, 
combined with the fact that $\hat{v}$ is a unit vector, we have:
\begin{align}
\|C(w_0)^T \hat{v}\|^2 = &\|[C(w_0) - C(w)]^T \hat{v}\|^2
= \sum_{i=1}^m ([\nabla c_i(w_0) - \nabla c_i(w)]^T \hat{v})^2 \nonumber \\
\le & \sum_{i=1}^m \|\nabla c_i(w_0) - \nabla c_i(w)\|^2 \|\hat{v}\|^2
\le \sum_{i=1}^m \beta_i^2\|w_0 - w\|^2 \label{C4_constraint}
\end{align}
Combining Eq.(\ref{C3_constraint}) and Eq.(\ref{C4_constraint}), and the definition of $R$, 
we concludes the proof.
\end{proof}

\begin{lemma} \label{lem::tangent}
Suppose the constraints $\{c_i\}_{i=1}^m$ are $\beta_i$-smooth, 
and \nameCQ holds for all $w\in \mathcal{W}$.
Let $\sum_{i=1}^m\frac{\beta_i^2}{ \alpha^2_c} = \frac{1}{R^2}$,
for any $w, w_0 \in \mathcal{W}$, let $\mathcal{T}_0 = \mathcal{T}(w_0)$.
Then for all $\hat{v} \in \mathcal{T}^c(w)$ so that $\|\hat{v}\| = 1$, we have 
\begin{equation}
\|P_{\mathcal{T}_0} \cdot \hat{v}\| \le \frac{\|w-w_0\|}{R}
\end{equation}
\end{lemma}

\begin{proof}
By definition of projection, clearly, we have
$P_{\mathcal{T}_0} \cdot \hat{v} + P_{\mathcal{T}^c_0} \cdot \hat{v} = \hat{v}$.
Since $\hat{v} \in \mathcal{T}^c(w)$, without loss of generality, assume 
$\hat{v} = \sum_{i=1}^m \lambda_i \nabla c_i(w)$.
Define $\tilde{d} = \sum_{i=1}^m \lambda_i \nabla c_i(w_0)$, 
clearly $\tilde{d} \in \mathcal{T}^c_0$. Since projection gives the closest point in subspace, 
we have:
\begin{align} \label{C5_constraint}
\|P_{\mathcal{T}_0} \cdot \hat{v}\| = &\| P_{\mathcal{T}^c_0} \cdot \hat{v} - \hat{v}\|
\le  \|\tilde{d} - \hat{v}\| \nonumber \\
\le&  \sum_{i=1}^m \lambda_i \|\nabla c_i(w_0) - \nabla c_i(w)\|
\le \sum_{i=1}^m \lambda_i \beta_i \|w_0 - w\|
\end{align}
On the other hand, let $\lambda = (\lambda_1, \cdots, \lambda_m)^T$, 
we know $C(w) \lambda = \hat{v}$, thus:
\begin{equation}
\lambda  = C(w)^{\dagger}\hat{v} = (C(w)^TC(w))^{-1}C(w)^T\hat{v}
\end{equation}
Therefore, by \nameCQ and the fact $\hat{v}$ is unit vector, we know: $\|\lambda\| \le \frac{1}{\alpha_c}$.
Combined with Eq.(\ref{C5_constraint}), we finished the proof.
\end{proof}

Using the previous lemmas, we can then prove that: starting from any point $w_0$ on constraint manifold, the result of
adding any small vector $v$ and then projected back to feasible set, is not very different from 
the result of adding $P_{\mathcal{T}(w_0)} v$.

\begin{lemma} \label{lem::projection_distance}
Suppose the constraints $\{c_i\}_{i=1}^m$ are $\beta_i$-smooth, 
and \nameCQ holds for all $w\in \mathcal{W}$.
Let $\sum_{i=1}^m\frac{\beta_i^2}{ \alpha^2_c} = \frac{1}{R^2}$,
for any $w_0 \in \mathcal{W}$, let $\mathcal{T}_0 = \mathcal{T}(w_0)$.
Then let $w_1 = w_0 + \eta \hat{v}$, and 
$w_2  = w_0 + \eta P_{\mathcal{T}_0}\cdot \hat{v}$, 
where $\hat{v} \in \mathbb{S}^{d-1}$ is a unit vector.
Then, we have:
\begin{equation}
\|\Pi_{\mathcal{W}}(w_1) - w_2\| \le \frac{4\eta^2}{R}
\end{equation}
Where projection $\Pi_{\mathcal{W}}(w)$ is defined as the closet point to $w$ on feasible set $\mathcal{W}$.
\end{lemma}

\begin{proof}
First, note that $\|w_1 - w_0\| = \eta$, and 
by definition of projection, there must exist a project $\Pi_{\mathcal{W}}(w)$ inside the ball
$\mathbb{B}_\eta (w_1) = \{w~|~ \|w-w_1\| \le \eta \}$.

Denote $u_1 = \Pi_{\mathcal{W}}(w_1)$, and clearly $u_1 \in \mathcal{W}$.
we can formulate $u_1$ as the solution to following constrained optimization problems:
\begin{align}
&\min_u \quad \quad  \|w_1 - u\|^2 \\
&\text{s.t.} \quad \quad c_i(u) = 0, \quad \quad i=1, \cdots, m \nonumber
\end{align}
Since function $f(u) = \|w_1 - u\|^2$ and $c_i(u)$ are continuously differentiable by assumption, 
and the condition \nameCQ holds for all $w\in \mathcal{W}$ implies that LICQ holds for $u_1$. Therefore, by 
Karush-Kuhn-Tucker necessary conditions, 
we immediately know $(w_1 - u_1) \in \mathcal{T}^c(u_1)$.

Since $u_1 \in \mathbb{B}_\eta (w_1)$, we know $\|w_0 - u_1\| \le 2 \eta$, by Lemma \ref{lem::tangent}, 
we immediately have:
\begin{equation}
\|P_{\mathcal{T}_0} (w_1 - u_1)\| 
= \frac{\|P_{\mathcal{T}_0} (w_1 - u_1)\| }{\|w_1 - u_1\| }\|w_1 - u_1\| 
\le \frac{1}{R}\|w_0 - u_1\|\cdot \|w_1 - u_1\| \le \frac{2}{R} \eta^2
\end{equation}

Let $v_1 =  w_0 + P_{\mathcal{T}_0}(u_1 - w_0)$, 
we have:
\begin{align} \label{F1_constraint}
\|v_1 - w_2\|= &\|(v_1 - w_0) - (w_2-w_0)\| = \|P_{\mathcal{T}_0}(u_1 - w_0) - P_{\mathcal{T}_0}(w_1 - w_0)\|
\nonumber \\
= &\|P_{\mathcal{T}_0} (w_1 - u_1)\| \le \frac{2}{R} \eta^2
\end{align}

On the other hand by Lemma \ref{lem::normal_by_tangent}, we have:
\begin{equation} \label{F2_constraint}
\|u_1 - v_1\| = \|P_{\mathcal{T}^c_0} (u_1-w_0)\| \le \frac{1}{2R} \|u_1-w_0\|^2 \le \frac{2}{R} \eta^2
\end{equation}
Combining Eq.(\ref{F1_constraint}) and Eq.(\ref{F2_constraint}), we finished the proof.

\end{proof}

\subsection{Main Theorem}

Now we are ready to prove the main theorems. First we revise the definition of \name~in the constrained case.

\begin{definition}
\label{def:robustcondition_constraint}
A twice differentiable function $f(w)$ with constraints $c_i(w)$ is $(\alpha, \gamma, \epsilon, \delta)$-{\em\name}, if for any point $w$ one of the following is true
\begin{enumerate}
\item $\|\chi(w)\| \ge \epsilon$.
\item $\hat{v}^T  \mathfrak{M}(w) \hat{v}  \le -\gamma$ for some $\hat{v} \in \mathcal{T}(w)$, $\|\hat{v}\| = 1$
\item There is a local minimum $w^\star$ such that $\|w-w^\star\| \le \delta$, and for all $w'$ in the $2\delta$ neighborhood of $w^\star$, we have $\hat{v}^T  \mathfrak{M}(w') \hat{v}  \ge \alpha$ for all $\hat{v} \in \mathcal{T}(w')$, $\|\hat{v}\| = 1$
\end{enumerate}
\end{definition}

Next, we prove a equivalent formulation for PSGD.

\begin{lemma}\label{PSGD_equivalent}
Suppose the constraints $\{c_i\}_{i=1}^m$ are $\beta_i$-smooth, 
and \nameCQ holds for all $w\in \mathcal{W}$. Furthermore, 
if function $f$ is $L$-Lipschitz, and the noise $\xi$ is bounded, then
running PSGD as in Eq.(\ref{PSGD_update}) is equivalent to running:
\begin{equation}\label{PSGD_update_equivalent}
	w_t = w_{t-1} -  \eta \cdot (\chi(w_{t-1}) + P_{\mathcal{T}(w_{t-1})} \xi_{t-1}) + \iota_{t-1}
\end{equation}
where $\iota$ is the correction for projection, and $\|\iota\| \le \tlO(\eta^2)$.
\end{lemma}

\begin{proof}
Lemma~\ref{PSGD_equivalent} is a direct corollary of Lemma~\ref{lem::projection_distance}.
\end{proof}

The intuition behind this lemma is that:
when $\{c_i\}_{i=1}^m$ are smooth and \nameCQ holds for all $w\in \mathcal{W}$, 
then the constraint manifold has bounded curvature every where.
Then, if we only care about first order behavior, it's well-approximated
by the local dynamic in tangent plane, up to some second-order correction.

Therefore, by Eq.(\ref{PSGD_update_equivalent}), we see locally it's not much different from the unconstrainted case Eq.(\ref{SGD_update}) up to some negeligable correction.
In the following analysis, we will always use formula Eq.(\ref{PSGD_update_equivalent}) as the update equation for PSGD.

Since most of following proof bears a lot similarity as in unconstrained case, we only pointed out the essential steps in our following proof.

\begin{theorem} [Main Theorem for Equality-Constrained Case]
Suppose a function $f(w):\R^d\to \R$ with constraints $c_i(w):\R^d\to \R$ is $(\alpha, \gamma, \epsilon, \delta)$-\name, and has a stochastic gradient oracle with radius at most $Q$, also satisfying $\E\xi = 0$ and $\E \xi\xi^T = \sigma^2I$. Further, suppose the function 
function $f$ is $B$-bounded, $L$-Lipschitz, $\beta$-smooth, and has $\rho$-Lipschitz Hessian,
and the constraints $\{c_i\}_{i=1}^m$ is $L_i$-Lipschitz, $\beta_i$-smooth, and has $\rho_i$-Lipschitz Hessian.
Then there exists a threshold $\eta_{\max} = \tilde{\Theta}(1)$, so that for any $\zeta>0$, and
for any $\eta \le \eta_{\max} / \max\{1, \log (1/\zeta)\}$,
with probability at least $1-\zeta$ in $t = \tlO(\eta^{-2}\log (1/\zeta))$ iterations, PSGD outputs a point $w_t$ that is $\tlO(\sqrt{\eta\log(1/\eta\zeta)})$-close to some local minimum $w^\star$.
\label{thm:constrainedmain}
\end{theorem}

First, we proof the assumptions in main theorem implies the smoothness conditions for $\mathfrak{M}(w)$, $\mathfrak{N}(w)$ and $\nabla^2_{ww} \mathcal{L}(w, \lambda^*(w'))$.

\begin{lemma}\label{lem:constrainedsmooth}
Under the assumptions of Theorem~\ref{thm:constrainedmain},
there exists $\beta_M, \beta_N, \rho_M, \rho_N, \rho_L$ polynomial related to 
$B, L, \beta, \rho, \frac{1}{\alpha_c}$ and $\{L_i, \beta_i, \rho_i\}_{i=1}^m$ so that:
\begin{enumerate}
\item $\|\mathfrak{M}(w)\| \le \beta_M$ and $\|\mathfrak{N}(w)\| \le \beta_N$ for all $w \in \mathcal{W}$.
\item $\mathfrak{M}(w)$ is $\rho_M$-Lipschitz, and $\mathfrak{N}(w)$ is $\rho_N$-Lipschitz, and
$\nabla^2_{ww} \mathcal{L}(w, \lambda^*(w'))$ is $\rho_L$-Lipschitz for all $w' \in \mathcal{W}$. 
\end{enumerate}
\end{lemma}

\begin{proof}
By definition of $\mathfrak{M}(w)$, $\mathfrak{N}(w)$ and $\nabla^2_{ww} \mathcal{L}(w, \lambda^*(w'))$, 
the above conditions will holds if there exists $B_\lambda, L_\lambda, \beta_\lambda$ bounded by 
$\tlO(1)$, so that $\lambda^*(w)$ is $B_\lambda$-bounded, $L_\lambda$-Lipschitz, and $\beta_\lambda$-smooth.

By definition Eq.(\ref{Lambda_star}), we have:
\begin{equation} 
\lambda^*(w)  = C(w)^{\dagger}\nabla f(w) = (C(w)^TC(w))^{-1}C(w)^T\nabla f(w)
\end{equation}
Because $f$ is $B$-bounded, $L$-Lipschitz, $\beta$-smooth, and its Hessian is $\rho$-Lipschitz, thus, eventually, we only need to prove that there exists $B_c, L_c, \beta_c$ bounded by $\tlO(1)$, so that the pseudo-inverse $C(w)^{\dagger}$ is $B_c$-bounded, $L_c$-Lipschitz, and $\beta_c$-smooth.

Since \nameCQ holds for all feasible points, we immediately have: $\|C(w)^{\dagger}\| \le\frac{1}{\alpha_c}$, thus bounded. For simplicity, in the following context we use $C^\dagger$ to represent $C^\dagger(w)$ without ambiguity.
By some calculation of linear algebra, we have the derivative of pseudo-inverse:
\begin{align}
&\frac{\partial C(w)^{\dagger}}
{\partial w_i}=-C^{\dagger}
\frac{\partial C(w)}{\partial w_i}C^{\dagger}
+C^{\dagger}[C^{\dagger}]^T
\frac{\partial C(w)^T}{\partial w_i}(I-C C^{\dagger} )
\end{align}
Again, \nameCQ holds implies that derivative of pseudo-inverse is well-defined for every feasible point.
Let tensor $E(w), \tilde{E}(w)$ to be the derivative of $C(w), C^\dagger(w)$, which is defined as:
\begin{equation}
[E(w)]_{ijk} =  \frac{\partial [C(w)]_{ik}}{\partial w_j} \quad \quad
[\tilde{E}(w)]_{ijk} =  \frac{\partial [C(w)^\dagger]_{ik}}{\partial w_j}
\end{equation}
Define the transpose of a 3rd order tensor $E^T_{i,j,k} = E_{k,j,i}$, then we have
\begin{equation} \label{derivative_pseudo_inverse}
\tilde{E}(w)
=-[E(w)](C^{\dagger},I,C^{\dagger})
+[E(w)^T](C^{\dagger}[C^{\dagger}]^T, I, (I-C C^{\dagger} ))
\end{equation}
where by calculation $[E(w)](I,I,e_i) = \nabla^2 c_i(w)$.

Finally, since $C(w)^\dagger$ and $\nabla^2 c_i(w)$ are bounded by $\tlO(1)$, 
by Eq.(\ref{derivative_pseudo_inverse}), we know $\tilde{E}(w)$
is bounded, that is $C(w)^{\dagger}$ is Lipschitz. 
Again, since both $C(w)^\dagger$ and $\nabla^2 c_i(w)$ are bounded, Lipschitz, 
by Eq.(\ref{derivative_pseudo_inverse}), we know $\tilde{E}(w)$ is also $\tlO(1)$-Lipschitz.
This finishes the proof.

\end{proof}

From now on,
we can use the same proof strategy as unconstraint case. Below we list the corresponding lemmas and the essential steps that require modifications.

\begin{lemma} \label{thm::case1_constraint}
Under the assumptions of Theorem~\ref{thm:constrainedmain}, with notations in Lemma~\ref{lem:constrainedsmooth}, for any point with $\|\chi(w_0)\| \ge \sqrt{2\eta\sigma^2\beta_{M} (d-m)}$ where $\sqrt{2\eta\sigma^2\beta_{M} (d-m)} < \epsilon$, after one iteration we have:
\begin{equation}
	\E f(w_1) - f(w_{0}) \le - \tlOmega(\eta^2)
\end{equation}
\end{lemma} 

\begin{proof}
Choose $\eta_{\max} < \frac{1}{\beta_{M}}$, and also small enough, then by update equation Eq.(\ref{PSGD_update_equivalent}), we have:
\begin{align}
	\E f(w_1) -  f(w_{0}) &\le \chi(w_{0})^T \E(w_1-w_{0}) + \frac{\beta_{M}}{2}\E\|w_1-w_{0}\|^2 \nonumber \\
	& \le -(\eta - \frac{\beta_{M}\eta^2}{2})\|\chi(w_{0})\|^2 + \frac{\eta^2\sigma^2 \beta_{M} (d-m)}{2}
	+ \tlO(\eta^{2})\|\chi(w_{0})\| + \tlO(\eta^3)
	\nonumber \\
	& \le -(\eta - \tlO(\eta^{1.5}) - \frac{\beta_{M}\eta^2}{2})\|\chi(w_{0})\|^2 + \frac{\eta^2\sigma^2 \beta_{M} (d-m)}{2} + \tlO(\eta^3)
	\nonumber \\
	&\le -\frac{\eta^2 \sigma^2 \beta_{M} d}{4}
\end{align}
Which finishes the proof.
\end{proof}

\begin{theorem}\label{thm::case3_constraint}
Under the assumptions of Theorem~\ref{thm:constrainedmain}, 
with notations in Lemma~\ref{lem:constrainedsmooth}, 
for any initial point $w_0$ that is $\tlO(\sqrt{\eta}) < \delta$ close to a local minimum $w^\star$, 
with probability at least $1-\zeta/2$, we have following holds simultaneously:
\begin{equation}
 \forall t\le \tlO(\frac{1}{\eta^2}\log \frac{1}{\zeta}), \quad \|w_{t} - w^\star\| \le \tlO(\sqrt{\eta\log \frac{1}{\eta\zeta}})<\delta 
\end{equation}
where $w^\star$ is the locally optimal point.
\end{theorem}

\begin{proof}
By calculus, we know 
\begin{align}
\chi(w_t) = & \chi(w^\star) + \int_{0}^{1}(\mathfrak{M} + \mathfrak{N})(w^\star + t(w_t - w^\star)) \mathrm{d}t \cdot (w_t - w^\star)
\end{align}

Let filtration $\mathfrak{F}_t = \sigma\{\xi_0, \cdots \xi_{t-1}\}$, and note $\sigma\{\Delta_0, \cdots, \Delta_t \} \subset \mathfrak{F}_t$, where $\sigma\{\cdot\}$ denotes the sigma field.
Let event $\mathfrak{E}_t = \{\forall \tau \le t, \|w_{\tau} - w^\star\| \le \mu\sqrt{\eta\log\frac{1}{\eta\zeta}} < \delta \}$, where $\mu$ is independent of $(\eta, \zeta)$, and will be specified later.

By Definition~\ref{def:robustcondition_constraint} of 
$(\alpha, \gamma, \epsilon, \delta)$-\name, we know $\mathfrak{M}(w)$ is locally $\alpha$-strongly convex restricted to its tangent space $\mathcal{T}(w)$.
in the $2\delta$-neighborhood of $w^\star$. If $\eta_{\max}$ is chosen small enough, by Remark
\ref{N_constraint} and Lemma \ref{lem::normal_by_tangent}, we have in addition:
\begin{align}
\chi(w_t)^T (w_t - w^\star)1_{\mathfrak{E}_t} &= (w_t - w^\star)^T \int_{0}^{1}(\mathfrak{M}+ \mathfrak{N})(w^\star + t(w_t - w^\star)) \mathrm{d}t \cdot (w_t - w^\star)1_{\mathfrak{E}_t} \nonumber \\ 
&\ge [\alpha \|w_t - w^\star\|^2 - \tlO(\|w_t - w^\star\|^3)]1_{\mathfrak{E}_t} \ge 0.5\alpha \|w_t - w^\star\|^21_{\mathfrak{E}_t}
\end{align}
Then, everything else follows almost the same as the proof of Lemma \ref{thm::case3}.
\end{proof}

\begin{lemma} \label{thm::case2_constraint}
Under the assumptions of Theorem~\ref{thm:constrainedmain}, with notations in Lemma~\ref{lem:constrainedsmooth},
for any initial point $w_0$ where $\|\chi(w_0)\| \le \tlO(\eta) < \epsilon$, and 
$\hat{v}^T  \mathfrak{M}(w_0) \hat{v}  \le -\gamma$ for some $\hat{v} \in \mathcal{T}(w)$, $\|\hat{v}\| = 1$, then 
there is a number of steps $T$ that depends on $w_0$ such that: 
\begin{equation}
	\E f(w_T) - f(w_0) \le - \tlOmega(\eta)
\end{equation}
The number of steps $T$ has a fixed upper bound $T_{max}$ that is independent of $w_0$ where $T \le T_{max} = O((\log (d-m))/\gamma\eta)$.
\end{lemma} 

Similar to the unconstrained case, we show this by a coupling sequence. Here the sequence we construct will only walk on the tangent space, by Lemmas in previous subsection, we know this is not very far from the actual sequence. We first define and characterize the coupled sequence in the following lemma:

\begin{lemma} \label{lem::case_Gaussian_constraint}
Under the assumptions of Theorem~\ref{thm:constrainedmain}, with notations in Lemma~\ref{lem:constrainedsmooth}.
Let $\tilde{f}$ defined as local second-order approximation of $f(x)$ around $w_0$
in tangent space $\mathcal{T}_0 =\mathcal{T}(w_0)$:
\begin{equation}\label{def_tilde_f_constraint}
\tilde{f}(w) \doteq f(w_0) + \chi(w_0)^T (w-w_0) + \frac{1}{2}(w-w_0)^T[P_{\mathcal{T}_0}^T\mathfrak{M}(w_0)P_{\mathcal{T}_0}](w-w_0)
\end{equation}
$\{\tilde{w}_t\}$ be the corresponding sequence generated by running SGD on function $\tilde{f}$, with $\tilde{w}_0 = w_0$, and noise projected to $\mathcal{T}_0$, (i.e. $\tilde{w}_t = \tilde{w}_{t-1} - \eta (\tilde{\chi}(\tilde{w}_{t-1}) + P_{\mathcal{T}_0}\xi_{t-1}$).
For simplicity, denote 
$\tilde{\chi}(w) = \nabla \tilde{f}(w)$, and $\widetilde{\mathfrak{M}} = P_{\mathcal{T}_0}^T\mathfrak{M}(w_0)P_{\mathcal{T}_0}$,
then we have analytically:
\begin{align}
&\tilde{\chi}(\tilde{w}_t)= (1-\eta\widetilde{\mathfrak{M}} )^t\tilde{\chi}(\tilde{w}_0) -\eta \widetilde{\mathfrak{M}} \sum_{\tau=0}^{t-1}(1-\eta\widetilde{\mathfrak{M}} )^{t-\tau-1}P_{\mathcal{T}_0}\xi_{\tau}\\
 &\tilde{w}_{t} - w_0 = -\eta \sum_{\tau = 0}^{t-1}(1-\eta \widetilde{\mathfrak{M}} )^\tau\tilde{\chi}(\tilde{w}_0) -\eta
	\sum_{\tau=0}^{t-1}(1-\eta\widetilde{\mathfrak{M}})^{t-\tau-1}P_{\mathcal{T}_0}\xi_{\tau}  \label{dif_x_constraint}
\end{align}
Furthermore, for any initial point $w_0$ where $\|\chi(w_0)\| \le \tlO(\eta) < \epsilon$, and 
$\min_{\hat{v} \in \mathcal{T}(w), \|\hat{v}\| = 1}\hat{v}^T  \mathfrak{M}(w_0) \hat{v}  = -\gamma_0$. There exist a $T \in \mathbb{N}$ satisfying:
\begin{equation}\label{choose_t_constraint}
 \frac{d-m}{\eta\gamma_0} \le \sum_{\tau=0}^{T-1}(1+\eta \gamma_0)^{2\tau} < \frac{3(d-m)}{\eta\gamma_0}
\end{equation}
with probability at least $1-\tlO(\eta^3)$, we have
following holds simultaneously for all $t\le T$:
\begin{equation}
\|\tilde{w}_t - w_0\| \le \tlO(\eta^{\frac{1}{2}}\log \frac{1}{\eta}); 
\quad\quad
\|\tilde{\chi}(\tilde{w}_t)\| \le \tlO(\eta^{\frac{1}{2}}\log \frac{1}{\eta})
\end{equation}
\end{lemma}

\begin{proof}
Clearly we have:
\begin{equation}\label{derivative_tilde_recursive_constraint}
\tilde{\chi}(\tilde{w}_t) = \tilde{\chi}(\tilde{w}_{t-1}) + \widetilde{\mathfrak{M}} (\tilde{w}_t - \tilde{w}_{t-1})
\end{equation}
and 
\begin{equation}\label{SGD_tilde_constraint}
\tilde{w}_t = \tilde{w}_{t-1} - \eta (\tilde{\chi}(\tilde{w}_{t-1}) + P_{\mathcal{T}_0}\xi_{t-1} )
\end{equation}
This lemma is then proved by a direct application of Lemma \ref{lem::case_Gaussian}.
\end{proof}

Then we show the sequence constructed is very close to the actual sequence.

\begin{lemma} \label{lem::saddle_and_maximum_constraint}
Under the assumptions of Theorem~\ref{thm:constrainedmain}, with notations in Lemma~\ref{lem:constrainedsmooth}.
Let $\{w_t\}$ be the corresponding sequence generated by running PSGD on function $f$.  
Also let $\tilde{f}$ and $\{\tilde{w}_t\}$ be defined as in Lemma \ref{lem::case_Gaussian_constraint}.
Then, for any initial point $w_0$ where $\|\chi(w_0)\|^2 \le \tlO(\eta) < \epsilon$, and $\min_{\hat{v} \in \mathcal{T}(w), \|\hat{v}\| = 1}\hat{v}^T  \mathfrak{M}(w_0) \hat{v}  = -\gamma_0$. Given the choice of $T$ as in Eq.(\ref{choose_t_constraint}), 
with probability at least $1-\tlO(\eta^2)$, we have following holds simultaneously for all $t\le T$:
\begin{align}
\|w_t - \tilde{w}_t\| \le \tlO( \eta\log^2\frac{1}{\eta});
\end{align}
\end{lemma}

\begin{proof}
First, we have update function of tangent gradient by:
\begin{align}
\chi(w_t) = & \chi(w_{t-1}) + \int_{0}^{1}\nabla \chi(w_{t-1} + t(w_t - w_{t-1})) \mathrm{d}t \cdot (w_t - w_{t-1}) \nonumber \\
= & \chi(w_{t-1}) + \mathfrak{M}(w_{t-1}) (w_t - w_{t-1}) + \mathfrak{N}(w_{t-1}) (w_t - w_{t-1})+ \theta_{t-1}
\end{align}
where the remainder:
\begin{equation}
\theta_{t-1} \equiv \int_{0}^{1}\left[\nabla \chi(w_{t-1} + t(w_t - w_{t-1})) - \nabla \chi(w_{t-1})\right] \mathrm{d}t \cdot (w_t - w_{t-1})
\end{equation}
Project it to tangent space $\mathcal{T}_0 =\mathcal{T}(w_0)$.
Denote $\widetilde{\mathfrak{M}} = P_{\mathcal{T}_0}^T\mathfrak{M}(w_0)P_{\mathcal{T}_0}$, 
and $\widetilde{\mathfrak{M}}'_{t-1} = P_{\mathcal{T}_0}^T[~\mathfrak{M}(w_{t_1}) -\mathfrak{M}(w_0)~] P_{\mathcal{T}_0}$. Then, 
we have:
\begin{align}
P_{\mathcal{T}_0}\cdot\chi(w_t)= & P_{\mathcal{T}_0}\cdot\chi(w_{t-1}) + P_{\mathcal{T}_0}(\mathfrak{M}(w_{t-1})+ \mathfrak{N}(w_{t-1})) (w_t - w_{t-1}) + P_{\mathcal{T}_0}\theta_{t-1} \nonumber  \\
= & P_{\mathcal{T}_0}\cdot\chi(w_{t-1}) 
+ P_{\mathcal{T}_0}\mathfrak{M}(w_{t-1})P_{\mathcal{T}_0} (w_t - w_{t-1})  \nonumber \\
&+ P_{\mathcal{T}_0}\mathfrak{M}(w_{t-1}) P_{\mathcal{T}^c_0}(w_t - w_{t-1})
+ P_{\mathcal{T}_0}\mathfrak{N}(w_{t-1}) (w_t - w_{t-1})
+ P_{\mathcal{T}_0}\theta_{t-1} \nonumber \\
= & P_{\mathcal{T}_0}\cdot\chi(w_{t-1}) 
+ \widetilde{\mathfrak{M}}(w_t - w_{t-1}) + \phi_{t-1}\label{derivative_recursive_constraint}
\end{align}
Where 
\begin{equation}
\phi_{t-1} =  [~\widetilde{\mathfrak{M}}'_{t-1}  + P_{\mathcal{T}_0}\mathfrak{M}(w_{t-1}) P_{\mathcal{T}^c_0}
+ P_{\mathcal{T}_0}\mathfrak{N}(w_{t-1})~] (w_t - w_{t-1})
+ P_{\mathcal{T}_0}\theta_{t-1}
\end{equation}
By Hessian smoothness, we immediately have: 
\begin{align} 
&\|\widetilde{\mathfrak{M}}'_{t-1}\| = \|\mathfrak{M}(w_{t_1}) -\mathfrak{M}(w_0)\| 
\le \rho_M \|w_{t-1} - w_0\| \le \rho_M (\|w_t - \tilde{w}_t\| + \|\tilde{w}_t - w_0\|) 
\label{H'_smooth_constraint} \\
&\|\theta_{t-1}\| \le \frac{\rho_M+\rho_N}{2} \|w_t - w_{t-1}\|^2 \label{theta_smooth_constraint}
\end{align}

Substitute the update equation of PSGD (Eq.(\ref{PSGD_update_equivalent})) into Eq.(\ref{derivative_recursive_constraint}), we have:
\begin{align}
&P_{\mathcal{T}_0}\cdot\chi(w_t) = P_{\mathcal{T}_0}\cdot\chi(w_{t-1}) -\eta\widetilde{\mathfrak{M}} (P_{\mathcal{T}_0}\cdot \chi(w_{t-1}) + P_{\mathcal{T}_0}\cdot P_{\mathcal{T}(w_{t-1})}\xi_{t-1} ) +\widetilde{\mathfrak{M}}\cdot\iota_{t-1} + \phi_{t-1} \nonumber \\
	&= (1-\eta\widetilde{\mathfrak{M}})P_{\mathcal{T}_0}\cdot\chi(w_{t-1}) - \eta \widetilde{\mathfrak{M}} P_{\mathcal{T}_0}\xi_{t-1} + \eta \widetilde{\mathfrak{M}}P_{\mathcal{T}_0}\cdot P_{\mathcal{T}^c(w_{t-1})}\xi_{t-1}
	 +\widetilde{\mathfrak{M}}\cdot\iota_{t-1} +\phi_{t-1} 
				\label{derivative_constraint}
\end{align}

Let $\Delta_t = P_{\mathcal{T}_0}\cdot \chi(w_t) - \tilde{\chi}(\tilde{w}_t)$ denote the difference of tangent gradient in $\mathcal{T}(w_0)$, then
from Eq.(\ref{derivative_tilde_recursive_constraint}), Eq.(\ref{SGD_tilde_constraint}), 
and Eq.(\ref{derivative_constraint})
we have:
\begin{align} \label{Delta_recursive_constraint}
&\Delta_t =  (1-\eta H) \Delta_{t-1} + \eta \widetilde{\mathfrak{M}}P_{\mathcal{T}_0}\cdot P_{\mathcal{T}^c(w_{t-1})}\xi_{t-1}
	 +\widetilde{\mathfrak{M}}\cdot\iota_{t-1} +\phi_{t-1} \\
& P_{\mathcal{T}_0} \cdot (w_t-w_0) - (\tilde{w}_t-w_0) =  -\eta \sum_{\tau = 0}^{t-1} \Delta_\tau 
 + \eta \sum_{\tau = 0}^{t-1} P_{\mathcal{T}_0}\cdot P_{\mathcal{T}^c(w_{\tau})}\xi_{\tau}
	 +\sum_{\tau = 0}^{t-1}\iota_{\tau}\label{dif_constraint_tangent} 
\end{align}

By Lemma \ref{lem::normal_by_tangent}, we know if $\sum_{i=1}^m \frac{\beta_i^2}{\alpha_c^2} = \frac{1}{R^2}$, then we have:
\begin{align}
\|P_{\mathcal{T}^c_0} (w_t-w_0)\| \le \frac{\|w_t-w_0\|^2}{2R}
\label{dif_constraint_normal}
\end{align}

Let filtration $\mathfrak{F}_t = \sigma\{\xi_0, \cdots \xi_{t-1}\}$, and note $\sigma\{\Delta_0, \cdots, \Delta_t \} \subset \mathfrak{F}_t$, where $\sigma\{\cdot\}$ denotes the sigma field. Also, let event $\mathfrak{K}_t = \{\forall \tau \le t, ~\|\tilde{\chi}(\tilde{w}_\tau)\| \le \tlO(\eta^{\frac{1}{2}}\log \frac{1}{\eta}), ~ 
\|\tilde{w}_\tau - w_0\| \le \tlO(\eta^{\frac{1}{2}}\log \frac{1}{\eta})\}$,
and denote $\Gamma_t = \eta \sum_{\tau = 0}^{t-1} P_{\mathcal{T}_0}\cdot P_{\mathcal{T}^c(w_{\tau})}\xi_{\tau}$, let
 $\mathfrak{E}_t = \{\forall \tau \le t, ~\|\Delta_{\tau}\| \le \mu_1 \eta\log^2\frac{1}{\eta}, 
\|\Gamma_\tau\| \le \mu_2 \eta\log^2\frac{1}{\eta}, 
\|w_{\tau} - \tilde{w}_{\tau}\| \le \mu_3 \eta\log^2\frac{1}{\eta}\}$ where $(\mu_1, \mu_2, \mu_3)$ are is independent of $(\eta, \zeta)$, and will be determined later. To prevent ambiguity in the proof, $\tilde{O}$ notation will not hide any dependence on $\mu$.
Clearly event $\mathfrak{K}_{t-1}\subset \mathfrak{F}_{t-1},  \mathfrak{E}_{t-1}\subset \mathfrak{F}_{t-1}$ thus independent of $\xi_{t-1}$.

Then, conditioned on event $\mathfrak{K}_{t-1} \cap \mathfrak{E}_{t-1}$, 
by triangle inequality, we have $\| w_\tau - w_0\| \le \tlO(\eta^{\frac{1}{2}}\log \frac{1}{\eta})$, for all $\tau \le t-1 \le T-1$. 
We then need to carefully bound the following bound each term in Eq.(\ref{Delta_recursive_constraint}).
We know $w_t - w_{t-1} = -  \eta \cdot (\chi(w_{t-1}) + P_{\mathcal{T}(w_{t-1})} \xi_{t-1}) + \iota_{t-1}$, and then
by Lemma \ref{lem::tangent} and Lemma \ref{lem::normal}, we have:
\begin{align}
\|\eta \widetilde{\mathfrak{M}}P_{\mathcal{T}_0}\cdot P_{\mathcal{T}^c(w_{t-1})}\xi_{t-1}\| 
&\le \tlO(\eta^{1.5} \log \frac{1}{\eta}) \nonumber \\
\|\widetilde{\mathfrak{M}}\cdot\iota_{t-1} \| 
&\le \tlO(\eta^2) \nonumber \\
\|[~\widetilde{\mathfrak{M}}'_{t-1}  + P_{\mathcal{T}_0}\mathfrak{M}(w_{t-1}) P_{\mathcal{T}^c_0}
+ P_{\mathcal{T}_0}\mathfrak{N}(w_{t-1})~] ( -  \eta \cdot \chi(w_{t-1}))\|
&\le \tlO(\eta^2\log^2\frac{1}{\eta}) \nonumber \\
\|[~\widetilde{\mathfrak{M}}'_{t-1}  + P_{\mathcal{T}_0}\mathfrak{M}(w_{t-1}) P_{\mathcal{T}^c_0}
+ P_{\mathcal{T}_0}\mathfrak{N}(w_{t-1})~] ( -  \eta P_{\mathcal{T}(w_{t-1})} \xi_{t-1})\|
&\le \tlO(\eta^{1.5}\log\frac{1}{\eta}) \nonumber \\
\|[~\widetilde{\mathfrak{M}}'_{t-1}  + P_{\mathcal{T}_0}\mathfrak{M}(w_{t-1}) P_{\mathcal{T}^c_0}
+ P_{\mathcal{T}_0}\mathfrak{N}(w_{t-1})~] \iota_{t-1}\|
&\le \tlO(\eta^{2}) \nonumber \\
\|P_{\mathcal{T}_0}\theta_{t-1}\|  &\le \tlO(\eta^2)
\end{align}


Therefore, abstractly,  conditioned on event $\mathfrak{K}_{t-1} \cap \mathfrak{E}_{t-1}$, we could write down
the recursive equation as:
\begin{equation}
	\Delta_t =  (1-\eta H) \Delta_{t-1} + A + B
\end{equation}
where $\|A\| \le \tlO(\eta^{1.5} \log \frac{1}{\eta})$ and $\|B\| \le \tlO(\eta^{2} \log^2 \frac{1}{\eta})$, 
and in addition, by independence, easy to check we also have $\E [(1-\eta H) \Delta_{t-1} A |\mathfrak{F}_{t-1}] = 0$. This is exactly the same case as in the proof of Lemma \ref{lem::saddle_and_maximum}. By the same argument of martingale and Azuma-Hoeffding, and by choosing $\mu_1$ large enough, we can prove
\begin{align}\label{EE_1}
&P\left( \mathfrak{E}_{t-1} \cap \left\{\|\Delta_t\| \ge \mu_1 \eta\log^2\frac{1}{\eta}\right\}\right)  \le \tlO(\eta^3)
\end{align}

On the other hand, for $\Gamma_t = \eta \sum_{\tau = 0}^{t-1} P_{\mathcal{T}_0}\cdot P_{\mathcal{T}^c(w_{\tau})}\xi_{\tau}$,
we have:
\begin{align}
\E[\Gamma_t 1_{\mathfrak{K}_{t-1}\cap \mathfrak{E}_{t-1}} |\mathfrak{F}_{t-1}]
&= \left[\Gamma_{t-1}+\eta\E[P_{\mathcal{T}_0}\cdot P_{\mathcal{T}^c(w_{t-1})}\xi_{t-1}|\mathfrak{F}_{t-1}]\right]1_{\mathfrak{K}_{t-1}\cap \mathfrak{E}_{t-1}} \nonumber \\
&= \Gamma_{t-1}1_{\mathfrak{K}_{t-1}\cap \mathfrak{E}_{t-1}} \le \Gamma_{t-1}1_{\mathfrak{K}_{t-2}\cap \mathfrak{E}_{t-2}}
\end{align}

Therefore, we have $\E [\Gamma_t1_{\mathfrak{K}_{t-1}\cap \mathfrak{E}_{t-1}} ~|~ \mathfrak{F}_{t-1}] \le \Gamma_{t-1}1_{\mathfrak{K}_{t-2}\cap \mathfrak{E}_{t-2}}$ which means 
$\Gamma_t1_{\mathfrak{K}_{t-1}\cap \mathfrak{E}_{t-1}}$ is a supermartingale.

We also know by Lemma \ref{lem::tangent}, with probability 1:
\begin{align}
&| \Gamma_t1_{\mathfrak{K}_{t-1}\cap \mathfrak{E}_{t-1}} - \E[\Gamma_t1_{\mathfrak{K}_{t-1}\cap \mathfrak{E}_{t-1}}~|~\mathfrak{F}_{t-1}] |
= |\eta P_{\mathcal{T}_0}\cdot P_{\mathcal{T}^c(w_{t-1})}\xi_{t-1}|\cdot 1_{\mathfrak{K}_{t-1}\cap \mathfrak{E}_{t-1}} \nonumber \\
\le & \tlO(\eta)\|w_{t-1} - w_0\|1_{\mathfrak{K}_{t-1}\cap \mathfrak{E}_{t-1}}
\le \tlO(\eta^{1.5}\log \frac{1}{\eta}) = c_{t-1}
\end{align}
By Azuma-Hoeffding inequality, with probability less than $\tlO(\eta^3)$, 
for $t\le T\le O(\log (d-m)/\gamma_0\eta)$:
\begin{equation}
\Gamma_t1_{\mathfrak{K}_{t-1}\cap \mathfrak{E}_{t-1}} - \Gamma_0\cdot1 > \tlO(1)\sqrt{\sum_{\tau=0}^{t-1}{c^2_\tau}}\log (\frac{1}{\eta}) = \tlO(\eta\log^2 \frac{1}{\eta})
\end{equation}
This means there exists some $\tilde{C}_2 = \tlO(1)$ so that:
\begin{equation}
	P\left(\mathfrak{K}_{t-1}\cap \mathfrak{E}_{t-1} \cap \left\{\|\Gamma_t\| \ge \tilde{C}_2\eta\log^2\frac{1}{\eta}\right\}\right) \le \tlO(\eta^3)
\end{equation}
by choosing $\mu_2>\tilde{C}_2$, we have:
\begin{equation}
	P\left(\mathfrak{K}_{t-1}\cap \mathfrak{E}_{t-1} \cap \left\{\|\Gamma_t\| \ge \mu_2 \eta\log^2\frac{1}{\eta}\right\}\right) \le \tlO(\eta^3)
\end{equation}
Therefore, combined with Lemma \ref{lem::case_Gaussian_constraint}, we have:
\begin{align} \label{EE_2}
P\left( \mathfrak{E}_{t-1} \cap \left\{\|\Gamma_t\| \ge \mu_2 \eta\log^2\frac{1}{\eta}\right\}\right) 
\le \tlO(\eta^3) + P(\overline{\mathfrak{K}}_{t-1}) \le \tlO(\eta^3)
\end{align}

Finally, conditioned on event $\mathfrak{K}_{t-1} \cap \mathfrak{E}_{t-1}$, if we have $\|\Gamma_t\| \le \mu_2 \eta\log^2\frac{1}{\eta}$, then by Eq.(\ref{dif_constraint_tangent}):
\begin{equation}
\|P_{\mathcal{T}_0} \cdot (w_t-w_0) - (\tilde{w}_t-w_0)\| \le \tlO\left  ((\mu_1 + \mu_2)\eta\log^2\frac{1}{\eta}\right  )
\end{equation}
Since $\|w_{t-1} - w_0\| \le \tlO(\eta^{\frac{1}{2}}\log \frac{1}{\eta})$, 
and $\|w_{t} - w_{t-1}\| \le \tlO(\eta)$, by Eq.(\ref{dif_constraint_normal}):
\begin{equation}
\|P_{\mathcal{T}^c_0} (w_t-w_0)\| \le \frac{\|w_t-w_0\|^2}{2R}
\le \tlO(\eta\log^2 \frac{1}{\eta})
\end{equation}
Thus:
\begin{align}
\|w_t - \tilde{w}_t\|^2 = &\|P_{\mathcal{T}_0} \cdot (w_t - \tilde{w}_t)\|^2 + \|P_{\mathcal{T}^c_0} \cdot (w_t - \tilde{w}_t)\|^2 \nonumber \\
=& \|P_{\mathcal{T}_0} \cdot (w_t-w_0) - (\tilde{w}_t-w_0)\|^2 + \|P_{\mathcal{T}^c_0} (w_t-w_0)\|^2
\le \tlO((\mu_1 + \mu_2)^2\eta^{2}\log^4\frac{1}{\eta})
\end{align}
That is there exist some $\tilde{C}_3=\tlO(1)$ so that $\|w_t - \tilde{w}_t\|
\le \tilde{C}_3(\mu_1 + \mu_2)\eta\log^2\frac{1}{\eta}$
Therefore, conditioned on event $\mathfrak{K}_{t-1} \cap \mathfrak{E}_{t-1}$, we have proved that if choose
$\mu_3>\tilde{C}_3(\mu_1 + \mu_2)$, then event
$\{\|w_t - \tilde{w}_t\| \ge \mu_3 \eta\log^2\frac{1}{\eta} \} \subset \{\|\Gamma_t\| \ge \mu_2 \eta\log^2\frac{1}{\eta}\}$. Then, combined this fact with Eq.(\ref{EE_1}), Eq.(\ref{EE_2}), we have proved:
\begin{equation}
P\left( \mathfrak{E}_{t-1} \cap \overline{ \mathfrak{E}}_{t}\right) \le \tlO(\eta^3)
\end{equation}
Because $P(\overline{ \mathfrak{E}}_{0}) =0$, and $T\le \tlO(\frac{1}{\eta})$, we have 
$P(\overline{ \mathfrak{E}}_{T}) \le \tlO(\eta^2)$, which concludes the proof.

\end{proof}

These two lemmas allow us to prove the result when the initial point is very close to a saddle point.

\begin{proof}[Proof of Lemma \ref{thm::case2_constraint}]
Combine Talyor expansion Eq.\ref{Taylor_eq_constraint} with Lemma \ref{lem::case_Gaussian_constraint}, Lemma \ref{lem::saddle_and_maximum_constraint}, we prove this Lemma by the same argument as in the proof of Lemma \ref{thm::case2}.
\end{proof}

Finally the main theorem follows.

\begin{proof} [Proof of Theorem \ref{thm:constrainedmain}]
By Lemma \ref{thm::case1_constraint}, Lemma \ref{thm::case2_constraint}, and Lemma \ref{thm::case3_constraint}, with the same argument as in the proof Theorem \ref{thm:sgdmain_unconstraint}, we easily concludes this proof.
\end{proof}
\clearpage

\section{Detailed Proofs for Section~\ref{sec:tensors}}

In this section we show two optimization problems (\ref{eq:findone}) and (\ref{eq:hardprob}) satisfy the $(\alpha,\gamma,\epsilon,\delta)$-\name~propery.

\subsection{Warm up: maximum eigenvalue formulation}
\label{sec:warmup}
Recall that we are trying to solve the optimization (\ref{eq:findone}), which we restate here.
\begin{align}
\max & \quad T(u,u,u,u), \\ 
\|u\|^2 &= 1. \nonumber
\end{align}
Here the tensor $T$ has orthogonal decomposition $T = \sum_{i=1}^d a_i^{\otimes 4}$. We first do a change of coordinates to work in the coordinate system specified by $(a_i)$'s (this does not change the dynamics of the algorithm). In particular, let $u = \sum_{i=1}^d x_i a_i$ (where $x\in \R^d$), then we can see $T(u,u,u,u) = \sum_{i=1}^d x_i^4$. Therefore let $f(x) = -\|x\|_4^4$, the optimization problem is equivalent to
\begin{align} \label{problem1_transformed}
\min &~~~~ f(x)\\
\text{s.t.} & ~~~~\|x\|^2_2 = 1 \nonumber
\end{align}

This is a constrained optimization, so we apply the framework developed in Section~\ref{sec:constrainedproblem}.

Let $c(x) = \|x\|_2^2 -1$. We first compute the Lagrangian
\begin{equation}
\mathcal{L}(x, \lambda) = f(x) -\lambda c(x) = -\|x\|_4^4 - \lambda (\|x\|_2^2 -1).
\end{equation}

Since there is only one constraint, and the gradient when $\|x\| = 1$ always have norm $2$, we know the set of constraints satisfy $2$-RLICQ. In particular, we can compute the correct value of Lagrangian multiplier $\lambda$, 

\begin{equation}
\lambda^*(x) = \arg\min_{\lambda} \|\nabla_x \mathcal{L}(x, \lambda)\| 
= \arg\min_{\lambda} \sum_{i=1}^d (2 x_i^3 + \lambda x_i)^2 = -2\|x\|_4^4
\end{equation}

Therefore, the gradient in the tangent space is equal to
\begin{align} \label{chi_1}
\chi(x) & = \nabla_x \mathcal{L}(x, \lambda) |_{(x, \lambda^*(x))} = \nabla f(x) -\lambda^*(x) \nabla c(x) \nonumber \\
&= -4(x_1^3, \cdots, x_d^3)^T -2 \lambda^*(x)( x_1,\cdots, x_d)^T\nonumber \\
&=4\left((x_1^2-\|x\|_4^4) x_1, \cdots, (x_d^2-\|x\|_4^4) x_d\right)
\end{align}

The second-order partial derivative of Lagrangian is equal to
\begin{align} \label{frakM_1}
\mathfrak{M}(x)
& = \nabla^2_{xx} \mathcal{L}(x, \lambda)|_{(x, \lambda^*(x))}=
\nabla^2 f(x) -\lambda^* (x)\nabla^2 c(x) \nonumber \\
&= -12 \text{diag}(x_1^2, \cdots, x_d^2)	-2 \lambda^*(x) I_d \nonumber \\
&= -12 \text{diag}(x_1^2, \cdots, x_d^2)	+ 4\|x\|_4^4 I_d
\end{align}

Since the variable $x$ has bounded norm, and the function is a polynomial, it's clear that the function itself is bounded and all its derivatives are bounded. Moreover, all the derivatives of the constraint are bounded. We summarize this in the following lemma.
\begin{lemma}
The objective function (\ref{eq:findone}) is bounded by $1$, its $p$-th order derivative is bounded by $O(\sqrt{d})$ for $p = 1,2,3$.
The constraint's $p$-th order derivative is bounded by $2$, for $p=1,2,3$. 
\end{lemma}

Therefore the function satisfy all the smoothness condition we need. Finally we show the gradient and Hessian of Lagrangian satisfy the $(\alpha,\gamma, \epsilon,\delta)$-\name~property. Note that we did not try to optimize the dependency with respect to $d$.

\begin{theorem} \label{thm:problem_1_strict_saddle}
The only local minima of 
optimization problem (\ref{eq:findone}) are $\pm a_i ~(i\in[d])$. Further it satisfy $(\alpha,\gamma, \epsilon,\delta)$-\name~for $\gamma = 7/d$, $\alpha = 3$ and $\epsilon,\delta = 1/\mbox{poly}(d)$.
\end{theorem}

In order to prove this theorem, we consider the transformed version
Eq.\ref{problem1_transformed}. We first need following two lemma for points around saddle point and local minimum respectively. We choose 
\begin{equation}\label{choice_1}
\epsilon_0=(10d)^{-4}, ~~\epsilon= 4\epsilon_0^2, ~~\delta = 2d\epsilon_0, ~~\BigC(x) = \{ i ~|  ~|x_i| >  \epsilon_0\}
\end{equation}
Where by intuition, $\BigC(x)$ is the set of coordinates whose value is relative large.

\begin{lemma}\label{lem:Problem1_case2}
Under the choice of parameters in Eq.(\ref{choice_1}),
suppose $\|\chi(x)\| \le \epsilon$, and $|\BigC(x)| \ge 2$. Then,
there exists $\hat{v} \in \mathcal{T}(x)$ and $\|\hat{v}\| = 1$, so that
$\hat{v}^T  \mathfrak{M}(x) \hat{v}  \le -7/d$.
\end{lemma}

\begin{proof}
Suppose $|\BigC(x)| = p$, and $2\le p \le d$.
Since $\|\chi(x)\| \le \epsilon = 4\epsilon_0^2$, by Eq.(\ref{chi_1}), we have for each $i \in [d]$, $|[\chi(x)]_i|  = 4|(x_i^2-\|x\|_4^4)x_i|\le 4\epsilon_0^2$.
Therefore, we have:
\begin{equation}\label{pp_1}
	\forall i \in \BigC(x), \quad \quad \quad |x_i^2-\|x\|_4^4| \le \epsilon_0
\end{equation}
and thus:
\begin{align}
&|\|x\|_4^4 - \frac{1}{p}| = |\|x\|_4^4 - \frac{1}{p}\sum_{i} x_i^2| \nonumber \\
\le &|\|x\|_4^4 - \frac{1}{p}\sum_{i \in \BigC(x)} x_i^2| + |\frac{1}{p}\sum_{i \in [d]-\BigC(x)} x_i^2|
\le \epsilon_0 + \frac{d-p}{p} \epsilon_0^2 \le 2 \epsilon_0
\end{align}
Combined with Eq.\ref{pp_1}, this means:
\begin{equation}
	\forall i \in \BigC(x), \quad \quad \quad |x_i^2-\frac{1}{p}| \le 3\epsilon_0
\end{equation}

Because of symmetry, WLOG we assume 
$\BigC(x) = \{1, \cdots, p\}$. Since $|\BigC(x)| \ge 2$, we can pick 
$\hat{v}=(a, b, 0, \cdots, 0)$. 
Here $a>0, b<0$, and $a^2+b^2=1$. 
We pick $a$ such that $a x_1+ b x_2=0$. The solution is the intersection of a radius $1$ circle and a line which passes $(0,0)$, which always exists. 
For this $\hat{v}$, we know $\|\hat{v}\| = 1$, and $\hat{v}^T x=0$ thus $\hat{v} \in \mathcal{T}(x)$.
We have:
\begin{align}
&\hat{v}^T  \mathfrak{M}(x) \hat{v}  
= -(12x_1^2+4\|x\|_4^4) a^2- (12x_2^2+4\|x\|_4^4)b^2 \nonumber \\
=& -8 x_1^2 a^2 - 8x_2^2 b^2 - 4(x_1^2-\|x\|_4^4))a^2 - 4(x_2^2-\|x\|_4^4))b^2
\nonumber \\
\le & -\frac{8}{p} + 24 \epsilon_0 + 4 \epsilon_0
\le -7/d
\end{align}
Which finishes the proof.
\end{proof}

\begin{lemma}\label{lem:Problem1_case1}
Under the choice of parameters in Eq.(\ref{choice_1}),
suppose $\|\chi(x)\| \le \epsilon$, and $|\BigC(x)| = 1$. Then,
there is a local minimum $x^\star$ such that $\|x-x^\star\| \le \delta$, and for all $x'$ in the $2\delta$ neighborhood of $x^\star$, we have $\hat{v}^T  \mathfrak{M}(x') \hat{v}  \ge 3$ for all $\hat{v} \in \mathcal{T}(x')$, $\|\hat{v}\| = 1$
\end{lemma}

\begin{proof}
WLOG, we assume $\BigC(x) = \{1\}$. Then, we immediately have for all $i>1$, 
$|x_i| \le \epsilon_0$, and thus:
\begin{equation}
1 \ge x_1^2 = 1-\sum_{i>1}x_i^2 \ge 1- d\epsilon_0^2 	
\end{equation} 
Therefore $x_1 \ge \sqrt{1-d\epsilon_0^2}$ or $x_1 \le -\sqrt{1-d\epsilon_0^2}$.
Which means $x_1$ is either close to $1$ or close to $-1$. By symmetry, we know WLOG, 
we can assume the case $x_1 \ge \sqrt{1-d\epsilon_0^2}$. Let $e_1 = (1,0,\cdots, 0)$, 
then we know:
\begin{equation}
\|x-e_1\|^2 \le (x_1-1)^2 + \sum_{i>1} x_i^2
\le 2d \epsilon_0^2 \le \delta^2
\end{equation}

Next, we show $e_1$ is a local minimum. According to Eq.\ref{frakM_1}, we know $\mathfrak{M}(e_1)$ is a diagonal matrix with $4$ on the diagonals except for the first diagonal entry (which is equal to $-8$), since $\mathcal{T}(e_1) = \text{span}\{e_2, \cdots, e_d\}$, we have:
\begin{align}
v^T\mathfrak{M}(e_1) v \ge 4 \|v\|^2 >0 \quad \quad 
\text{for all~} v\in \mathcal{T}(e_1), v\neq 0
\end{align}
Which by Theorem \ref{thm::second_sufficient} means $e_1$ is a local minimum.

Finally, 
denote $\mathcal{T}_1 = \mathcal{T}(e_1)$ be the tangent space of constraint manifold at $e_1$.
We know for all $x'$ in the $2\delta$ neighborhood of $e_1$, 
and for all $\hat{v} \in \mathcal{T}(x')$, $\|\hat{v}\| = 1$:
\begin{align}
 \hat{v}^T  \mathfrak{M}(x') \hat{v}  
 \ge &  \hat{v}^T  \mathfrak{M}(e_1) \hat{v}  - 
 |\hat{v}^T  \mathfrak{M}(e_1) \hat{v} - \hat{v}^T  \mathfrak{M}(x') \hat{v} | \nonumber\\
 = & 4\|P_{\mathcal{T}_1}\hat{v} \|^2 - 8\|P_{\mathcal{T}^c_1}\hat{v} \|^2
 - \|\mathfrak{M}(e_1) - \mathfrak{M}(x')\|\|\hat{v}\|^2 \nonumber \\
 = & 4 - 12\|P_{\mathcal{T}^c_1}\hat{v} \|^2 - \|\mathfrak{M}(e_1) - \mathfrak{M}(x')\|
\end{align}
By lemma \ref{lem::normal}, we know $\|P_{\mathcal{T}^c_1}\hat{v} \|^2 \le \|x'-e_1\|^2
\le 4\delta^2$. By Eq.(\ref{frakM_1}), we have:
\begin{align}
	&\|\mathfrak{M}(e_1) - \mathfrak{M}(x')\| \le  \|\mathfrak{M}(e_1) - \mathfrak{M}(x')\| 
	\le \sum_{(i,j)} |[\mathfrak{M}(e_1)]_{ij} - [\mathfrak{M}(x')]_{ij}| \nonumber \\
	\le&  \sum_{i} \left|-12 [e_1]^2_{i}+ 4\|e_1\|_4^4 - 12x^2_{i} + 4\|x\|_4^4\right|
	\le 64 d\delta
\end{align}
In conclusion, we have $\hat{v}^T  \mathfrak{M}(x') \hat{v} \ge 4- 48\delta^2-64 d\delta\ge 3$
which finishs the proof.
\end{proof}

Finally, we are ready to prove Theorem \ref{thm:problem_1_strict_saddle}.
\begin{proof}[Proof of Theorem \ref{thm:problem_1_strict_saddle}]

According to Lemma \ref{lem:Problem1_case2} and Lemma \ref{lem:Problem1_case1}, we immediately know the optimization problem satisfies $(\alpha,\gamma, \epsilon,\delta)$-\name.

The only thing remains to show is that the only local minima of 
optimization problem (\ref{eq:findone}) are $\pm a_i ~(i\in[d])$.
Which is equivalent to show that the only local minima of the transformed problem
is $\pm e_i ~(i\in [d])$, where $e_i = (0, \cdots, 0, 1, 0, \cdots, 0)$, where $1$ is on $i$-th coordinate.

By investigating the proof of Lemma \ref{lem:Problem1_case2} and Lemma \ref{lem:Problem1_case1}, we know these two lemmas actually hold for any small enough choice of $\epsilon_0$ satisfying $\epsilon_0 \le (10d)^{-4}$, by pushing $\epsilon_0 \rightarrow 0$, we know for any point satisfying $|\chi(x)| \le \epsilon \rightarrow 0$, 
if it is close to some local minimum, it must satisfy $1=|\BigC(x)| \rightarrow \supp(x)$. Therefore, we know the only possible local minima are $\pm e_i ~(i\in [d])$. In Lemma \ref{lem:Problem1_case1}, we proved $e_1$ is local minimum, by symmetry, we finishes the proof.
\end{proof}

\subsection{New formulation}
\label{sec:hardcase}
In this section we consider our new formulation (\ref{eq:hardprob}). We first restate the optimization problem here:

\begin{align}
\min  \quad &\sum_{i\ne j} T(u^{(i)},u^{(i)},u^{(j)},u^{(j)}),\\
\forall i\quad & \|u^{(i)}\|^2  = 1 . \nonumber
\end{align}
Note that we changed the notation for the variables from $u_i$ to $u^{(i)}$, because in later proofs we will often refer to the particular coordinates of these vectors.

Similar to the previous section, we perform a change of basis. The effect is equivalent to making $a_i$'s equal to basis vectors $e_i$ (and hence the tensor is equal to $T = \sum_{i=1}^d e_i^{\otimes 4}$.
After the transformation the equations become
\begin{align}\label{problem2_transformed}
\min &~~~~ \sum_{(i,j):i\neq j}h(u^{(i)}, u^{(j)} )\\
\text{s.t.} & ~~~~\|u^{(i)}\|^2 = 1 \quad\quad \forall i \in [d]\nonumber
\end{align}
Here $h(u^{(i)}, u^{(j)}) = \sum_{k=1}^d (u^{(i)}_k u^{(j)}_k)^2$, $(i,j) \in [d]^2$. We divided the objective function by $2$ to simplify the calculation.

Let $U\in \R^{d^2}$ be the concatenation of $\{u^{(i)}\}$ such that $U_{ij}=u^{(i)}_j$.
Let $c_i(U) = \|u^{(i)}\|^2 -1 $ and $f(U) = \frac{1}{2}\sum_{(i,j):i\neq j}h(u^{(i)}, u^{(j)})$.
We can then compute the Lagrangian
\begin{equation}
\mathcal{L}(U, \lambda) = f(U) -\sum_{i=1}^d\lambda_i c_i(U)
=\frac{1}{2}\sum_{(i,j):i\neq j}h (u^{(i)}, u^{(j)}) - \sum_{i=1}^d\lambda_i (\|u^{(i)}\|^2 -1 )
\end{equation}

The gradients of $c_i(U)$'s are equal to $(0, \cdots, 0, 2u^{(i)}, 0, \cdots, 0)^T$, all of these vectors are orthogonal to each other (because they have disjoint supports) and have norm $2$. Therefore the set of constraints satisfy $2$-RLICQ. We can then compute the Lagrangian multipiers $\lambda^*$ as follows

\begin{equation}
\lambda^*(U) = \arg\min_{\lambda} \|\nabla_U \mathcal{L}(U, \lambda)\| 
= \arg\min_{\lambda} 
4\sum_{i}\sum_k (\sum_{j:j\neq i}U^2_{jk}U_{ik} -  \lambda_i U_{ik})^2
\end{equation}
which gives:
\begin{equation}
\lambda_i^*(U) = \arg\min_{\lambda}\sum_k (\sum_{j:j\neq i}U^2_{jk}U_{ik} -  \lambda_i U_{ik})^2
= \sum_{j:j\neq i} h (u^{(j)}, u^{(i)} )\label{eq:computelambdastar}
\end{equation}

Therefore, gradient in the tangent space is equal to
\begin{align}
\chi(U) & = \nabla_U \mathcal{L}(U, \lambda) |_{(U, \lambda^*(U))} = \nabla f(U) -\sum_{i=1}^n\lambda_i^*(U) \nabla c_i(U).
\end{align}

The gradient is a $d^2$ dimensional vector (which can be viewed as a $d\times d$ matrix corresponding to entries of $U$), and we express this in a coordinate-by-coordinate way.
For simplicity of later proof, denote:
\begin{equation}
\psi_{ik}(U) = \sum_{j:j\neq i} [U^2_{jk} -  h (u^{(j)}, u^{(i)} ) ] = 
\sum_{j: j\neq i} [U_{jk}^2-\sum_{l=1}^d U_{il}^2 U_{jl}^2] 
\end{equation}
Then we have:
\begin{align}
[\chi(U)]_{ik} & = 2 (\sum_{j:j\neq i}U^2_{jk} -  \lambda^*_i(U)  )U_{ik}  \nonumber
 \\&= 2U_{ik}\sum_{j:j\neq i} (U^2_{jk} -  h (u^{(j)}, u^{(i)} )  )\nonumber
\\& =2 U_{ik}  \psi_{ik}(U)\label{chi_2}
\end{align}

Similarly we can compute the second-order partial derivative of Lagrangian as
\begin{align}
\mathfrak{M}(U)
=\nabla^2 f(U) -\sum_{i=1}^d\lambda_i^* \nabla^2 c_i(U).
\end{align}
The Hessian is a $d^2\times d^2$ matrix, we index it by $4$ indices in $[d]$. The entries are summarized below:
\begin{align}
[\mathfrak{M}(U)]_{ik,i'k'}
= & \left.\frac{\partial}{\partial U_{i'k'}} [\nabla_U \mathcal{L}(U, \lambda)]_{ik} \right|_{(U, \lambda^*(U))}
= \left.\frac{\partial}{\partial U_{i'k'}}  [2(\sum_{j:j\neq i}U^2_{jk} -  \lambda )U_{ik}]  \right|_{(U, \lambda^*(U))}\nonumber \\
= &
\begin{cases}
	2(\sum_{j:j\neq i}U^2_{jk} -  \lambda^*_i (U)) &\mbox{~if~} k=k', i=i'\\
	4 U_{i'k} U_{ik} & \mbox{~if~} k=k', i\neq i' \\
	0 &\mbox{~if~}  k \neq k' 
\end{cases} \nonumber \\
= &
\begin{cases}
	2\psi_{ik}(U) &\mbox{~if~} k=k', i=i' \\
	4 U_{i'k} U_{ik} & \mbox{~if~} k=k', i\neq i' \\
	0 &\mbox{~if~}  k \neq k' 
\end{cases} \label{frakM_2}
\end{align}

Similar to the previous case, it is easy to bound the function value and derivatives of the function and the constraints.
\begin{lemma}
The objective function (\ref{eq:hardprob}) and  $p$-th order derivative are all bounded by $\mbox{poly}(d)$ for $p = 1,2,3$. Each constraint's $p$-th order derivative is bounded by $2$, for $p=1,2,3$. 
\end{lemma}

Therefore the function satisfy all the smoothness condition we need. Finally we show the gradient and Hessian of Lagrangian satisfy the $(\alpha,\gamma,\epsilon,\delta)$-\name~property. Again we did not try to optimize the dependency with respect to $d$.

\begin{theorem}\label{thm:problem_2_strict_saddle}
Optimization problem (\ref{eq:hardprob}) has exactly $2^d \cdot d!$ local minimum that corresponds to permutation and sign flips of $a_i$'s. Further, it satisfy $(\alpha,\gamma,\epsilon,\delta)$-\name~for $\alpha = 1$ and $\gamma,\epsilon,\delta = 1/\mbox{poly}(d)$.
\end{theorem}

Again, in order to prove this theorem, we follow the same strategy: we consider the transformed version
Eq.\ref{problem2_transformed}. and first prove the following lemmas for points around saddle point and local minimum respectively. We choose 
\begin{equation}\label{choice_2}
\epsilon_0 = (10d)^{-6}
, ~~\epsilon= 2\epsilon_0^6, ~~\delta = 2d\epsilon_0, ~~\gamma=\epsilon_0^4/4, ~~\BigC(u) = \{ k ~|  ~|u_k| >  \epsilon_0\}
\end{equation}
Where by intuition, $\BigC(u)$ is the set of coordinates whose value is relative large.

\begin{lemma}\label{lem:Problem2_case2}
Under the choice of parameters in Eq.(\ref{choice_2}),
suppose $\|\chi(U)\| \le \epsilon$, and there exists $(i,j) \in [d]^2$ so that $\BigC(u^{(i)})
\cap \BigC(u^{(j)}) \neq \emptyset$. Then,
there exists $\hat{v} \in \mathcal{T}(U)$ and $\|\hat{v}\| = 1$, so that
$\hat{v}^T  \mathfrak{M}(U) \hat{v}  \le -\gamma$.
\end{lemma}

\begin{proof}
	Again, since $\|\chi(x)\| \le \epsilon = 2\epsilon_0^6$, by Eq.(\ref{chi_2}), we have for each $i \in [d]$, $|[\chi(x)]_{ik}|  = 2| U_{ik}  \psi_{ik}(U)|\le 2\epsilon_0^6$.
Therefore, have:
\begin{equation}\label{pp_2}
	\forall k \in \BigC(u^{(i)}), \quad \quad \quad |\psi_{ik}(U)| \le \epsilon^5_0
\end{equation}

Then, we prove this lemma by dividing it into three cases. Note in order to prove that there exists $\hat{v} \in \mathcal{T}(U)$ and $\|\hat{v}\| = 1$, so that
$\hat{v}^T  \mathfrak{M}(U) \hat{v}  \le -\gamma$; it suffices to find a vector $v \in \mathcal{T}(U)$ and $\|v\| \le 1$, so that
$v^T \mathfrak{M}(U) v  \le -\gamma$.

\paragraph{Case 1}: $|\BigC(u^{(i)})| \ge 2$, $|\BigC(u^{(j)})|\ge 2$, and 
$|\BigC(u^{(i)})\cap \BigC(u^{(j)})| \ge 2$.

WLOG, assume $\{1, 2\} \in \BigC(u^{(i)})\cap \BigC(u^{(j)})$, choose $v$ to be $v_{i1} = \frac{U_{i2}}{4}$, $v_{i2} = -\frac{U_{i1}}{4}$, $v_{j1} = \frac{U_{j2}}4$ and $v_{j2} = - \frac{U_{j1}}4$. All other entries of $v$ are zero.
Clearly $v \in \mathcal{T}(U)$, and $\|v\|\le 1$. On the other hand, we know $\mathfrak{M}(U)$ restricted to these 4 coordinates $(i1, i2, j1, j2)$ is

\begin{equation}
\left(\begin{array}{cccc}
2\psi_{i1}(U) & 0 & 4U_{i1}U_{j1} & 0 \\ 
0 & 2\psi_{i2}(U) & 0 &  4U_{i2}U_{j2} \\ 
 4U_{i1}U_{j1} & 0 & 2\psi_{j1}(U) & 0 \\ 
0 &  4U_{i2}U_{j2} & 0 & 2\psi_{j2}(U)
\end{array} \right)
\end{equation}
By Eq.(\ref{pp_2}), we know all diagonal entries are $\le 2\epsilon_0^5$. 

If $U_{i1}U_{j1}U_{i2}U_{j2}$ is negative, we have the
quadratic form:
\begin{align}
	v^T\mathfrak{M}(U) v = & U_{i1}U_{j1}U_{i2}U_{j2}+\frac{1}{8}[U_{i2}^2\psi_{i1}(U)+
U_{i1}^2\psi_{i2}(U)
+U_{j2}^2\psi_{j1}(U)+
U_{j1}^2\psi_{j2}(U)] \nonumber \\
\le & -\epsilon_0^4 + \epsilon_0^5 \le -\frac{1}{4}\epsilon^4_0 = -\gamma
\end{align}
If $U_{i1}U_{j1}U_{i2}U_{j2}$ is positive we just swap the sign of the first two coordinates $v_{i1} = -\frac{U_{i2}}2$, $v_{i2} = \frac{U_{i1}}2$ and the above argument would still holds.

\paragraph{Case 2}: $|\BigC(u^{(i)})| \ge 2$, $|\BigC(u^{(j)})|\ge 2$, and 
$|\BigC(u^{(i)})\cap \BigC(u^{(j)})| = 1$.

WLOG, assume $\{1, 2\} \in \BigC(u^{(i)})$ and $ \{1, 3\}\in \BigC(u^{(j)})$, choose $v$ to be 
$v_{i1} = \frac{U_{i2}}{4}$, $v_{i2} = -\frac{U_{i1}}{4}$, $v_{j1} = \frac{U_{j3}}{4}$ and $v_{j3} = - \frac{U_{j1}}{4}$.
All other entries of $v$ are zero.
Clearly $v \in \mathcal{T}(U)$ and $\|v\|\le 1$. On the other hand, we know $\mathfrak{M}(U)$ restricted to these 4 coordinates $(i1, i2, j1, j3)$ is

\begin{equation}
\left(\begin{array}{cccc}
2\psi_{i1}(U) & 0 & 4U_{i1}U_{j1} & 0 \\ 
0 & 2\psi_{i2}(U) & 0 & 0 \\ 
4U_{i1}U_{j1} & 0 & 2\psi_{j1}(U) & 0 \\ 
0 & 0 & 0 & 2\psi_{j3}(U)
\end{array} \right)
\end{equation}
By Eq.(\ref{pp_2}), we know all diagonal entries are $\le 2\epsilon_0^5$. 
If $U_{i1}U_{j1}U_{i2}U_{j3}$ is negative, we have the
quadratic form:
\begin{align}
	v^T\mathfrak{M}(U) v = & \frac{1}{2}U_{i1}U_{j1}U_{i2}U_{j3}+\frac{1}{8}[U_{i2}^2\psi_{i1}(U)+
U_{i1}^2\psi_{i2}(U)
+U_{j3}^2\psi_{j1}(U)+
U_{j1}^2\psi_{j3}(U)] \nonumber \\
\le & -\frac{1}{2}\epsilon_0^4 + \epsilon_0^5 \le -\frac{1}{4}\epsilon^4_0 = -\gamma
\end{align}
If $U_{i1}U_{j1}U_{i2}U_{j3}$ is positive we just swap the sign of the first two coordinates $v_{i1} = -\frac{U_{i2}}2$, $v_{i2} = \frac{U_{i1}}2$ and the above argument would still holds.

\paragraph{Case 3}: Either $|\BigC(u^{(i)})| =1 $ or $|\BigC(u^{(j)})| =1 $.

WLOG, suppose $|\BigC(u^{(i)})| =1$, and $\{1\}= \BigC(u^{(i)}) $, we know:
\begin{equation}
	| (u^{(i)}_1)^2 -1| \le (d-1)\epsilon_0^2
\end{equation}
On the other hand, since $\BigC(u^{(i)})\cap \BigC(u^{(j)})\neq \emptyset$, 
we have $\BigC(u^{(i)})\cap \BigC(u^{(j)}) = \{1\}$, and thus:
\begin{equation}
|\psi_{j1}(U)| = |\sum_{i': i'\neq j} U^2_{i'1} -  \sum_{i':i'\neq j} h (u^{(i')}, u^{(j)} ) |
\le \epsilon_0^5
\end{equation}
Therefore, we have:
\begin{align}
\sum_{i':i'\neq j} h (u^{(i')}, u^{(j)} ) \ge \sum_{i': i'\neq j} U^2_{i'1} - \epsilon_0^5
\ge U^2_{i1} - \epsilon_0^5 \ge 1-d\epsilon_0^2
\end{align}
and
\begin{align}
	\sum_{k=1}^d \psi_{jk}(U) = &\sum_{i': i'\neq j}\sum^d_{k=1} U^2_{i'k} -  d\sum_{i':i'\neq j} h (u^{(i')}, u^{(j)} ) \nonumber \\
	\le & d-1 - d(1-d\epsilon_0^2)  = -1 + d^2 \epsilon_0^2
\end{align}
Thus, we know, there must exist some $ k' \in [d]$, so that $\psi_{jk'}(U) \le -\frac{1}{d}
+ d\epsilon_0^2$. 
This means we have ``large'' negative entry on the diagonal of $\mathfrak{M}$. 
Since $|\psi_{j1}(U)| \le \epsilon_0^5$, we know $k'\neq 1$. WLOG, suppose $k'=2$, we have $|\psi_{j2}(U)| > \epsilon_0^5$, thus $|U_{j2}| \le \epsilon_0$.

Choose $v$ to be 
$v_{j1} = \frac{U_{j2}}{2}$, $v_{j2} = -\frac{U_{j1}}{2}$.
All other entries of $v$ are zero.
Clearly $v \in \mathcal{T}(U)$ and $\|v\|\le 1$. On the other hand, we know $\mathfrak{M}(U)$ restricted to these 2 coordinates $(j1, j2)$ is
\begin{equation}
\left(\begin{array}{cc}
2\psi_{j1}(U) & 0 \\
0 & 2\psi_{j2}(U) \\
\end{array} \right)
\end{equation}
We know $|U_{j1}| > \epsilon_0$, $|U_{j2}| \le \epsilon_0$, $|\psi_{j1}(U)|
\le \epsilon_0^5$, and $\psi_{j2}(U) \le -\frac{1}{d} + d\epsilon_0^2$. 
Thus:
\begin{align}
	v^T\mathfrak{M}(U) v = & \frac{1}{2}\psi_{j1}(U)U_{j2}^2+ \frac{1}{2}\psi_{j2}(U)U_{j1}^2 \nonumber \\
\le & \epsilon_0^7 - (\frac{1}{d} - d\epsilon_0^2)\epsilon_0^2 
\le -\frac{1}{2d}\epsilon_0^2  \le -\gamma
\end{align}
Since by our choice of $v$, we have $\|v\|\le 1$, we can choose $\hat{v} = v/\|v\|$, and immediately have $\hat{v} \in \mathcal{T}(U)$ and $\|\hat{v}\| = 1$, and
$\hat{v}^T  \mathfrak{M}(U) \hat{v}  \le -\gamma$.
\end{proof}

\begin{lemma}\label{lem:Problem2_case1}
Under the choice of parameters in Eq.(\ref{choice_2}),
suppose $\|\chi(U)\| \le \epsilon$, and for any $(i,j) \in [d]^2$ we have $\BigC(u^{(i)})
\cap \BigC(u^{(j)}) = \emptyset$. Then,
there is a local minimum $U^\star$ such that $\|U-U^\star\| \le \delta$,
and for 
all $U'$ in the $2\delta$ neighborhood of $U^\star$, we have $\hat{v}^T  \mathfrak{M}(U') \hat{v}  \ge 1$ for all $\hat{v} \in \mathcal{T}(U')$, $\|\hat{v}\| = 1$
\end{lemma}

\begin{proof}
WLOG, we assume $\BigC(u^{(i)}) = \{i\}$ for $i=1, \cdots, d$. Then, we immediately have:
\begin{equation}
	|u^{(i)}_j| \le \epsilon_0, \quad\quad | (u^{(i)}_i)^2 -1| \le (d-1)\epsilon_0^2, \quad\quad
	\forall (i, j)\in[d]^2, j\neq i
\end{equation} 
Then $u^{(i)}_i \ge \sqrt{1-d\epsilon_0^2}$ or $u^{(i)}_i \le -\sqrt{1-d\epsilon_0^2}$.
Which means $u^{(i)}_i$ is either close to $1$ or close to $-1$. By symmetry, we know WLOG, 
we can assume the case $u^{(i)}_i \ge \sqrt{1-d\epsilon_0^2}$ for all $i\in[d]$.

Let $V\in \mathbb{R}^{d^2}$ be the concatenation of $\{e_1, e_2, \cdots, e_d\}$, then we have:
\begin{equation}
	\|U-V\|^2 = \sum_{i=1}^d \|u^{(i)} - e_i\|^2\le 2 d^2 \epsilon_0^2 \le \delta^2
\end{equation}

Next, we show $V$ is a local minimum. According to Eq.\ref{frakM_2}, we know $\mathfrak{M}(V)$ is a diagonal matrix with $d^2$ entries: 
\begin{align}
	[\mathfrak{M}(V)]_{ik,ik} = 2\psi_{ik}(V)  = 2\sum_{j: j\neq i} [V_{jk}^2-\sum_{l=1}^d V_{il}^2 V_{jl}^2] =
	\begin{cases}
		2  &\mbox{~if~} i\neq k\\
		0  &\mbox{~if~} i=k
	\end{cases}
\end{align}
We know the unit vector in the direction that corresponds to $[\mathfrak{M}(V)]_{ii,ii}$ is 
not in the tangent space $\mathcal{T}(V)$ for all $i\in[d]$. Therefore, for any $v \in \mathcal{T}(V)$, we have
\begin{align}
v^T\mathfrak{M}(e_1) v \ge 2 \|v\|^2 >0 \quad \quad 
\text{for all~} v\in \mathcal{T}(V), v\neq 0
\end{align}
Which by Theorem \ref{thm::second_sufficient} means $V$ is a local minimum.

Finally, 
denote $\mathcal{T}_V = \mathcal{T}(V)$ be the tangent space of constraint manifold at $V$.
We know for all $U'$ in the $2\delta$ neighborhood of $V$, 
and for all $\hat{v} \in \mathcal{T}(x')$, $\|\hat{v}\| = 1$:
\begin{align}
 \hat{v}^T  \mathfrak{M}(U') \hat{v}  
 \ge &  \hat{v}^T  \mathfrak{M}(V) \hat{v}  - 
 |\hat{v}^T  \mathfrak{M}(V) \hat{v} - \hat{v}^T  \mathfrak{M}(U') \hat{v} | \nonumber\\
 = & 2\|P_{\mathcal{T}_V}\hat{v} \|^2 
 - \|\mathfrak{M}(V) - \mathfrak{M}(U')\|\|\hat{v}\|^2 \nonumber \\
 = & 2 - 2\|P_{\mathcal{T}^c_V}\hat{v} \|^2 - \|\mathfrak{M}(V) - \mathfrak{M}(U')\|
\end{align}
By lemma \ref{lem::normal}, we know $\|P_{\mathcal{T}^c_V}\hat{v} \|^2
\le \|U'-V\|^2 \le 4\delta^2$. By Eq.(\ref{frakM_2}), we have:
\begin{align}
	&\|\mathfrak{M}(V) - \mathfrak{M}(U')\| \le  \|\mathfrak{M}(V) - \mathfrak{M}(U')\| 
	\le \sum_{(i,j,k)} |[\mathfrak{M}(V)]_{ik,jk} - [\mathfrak{M}(U')]_{ik,jk}| \le  100d^3\delta
\end{align}
In conclusion, we have $\hat{v}^T  \mathfrak{M}(U') \hat{v} \ge 2- 8\delta^2-100d^3\delta\ge 1$
which finishs the proof.
\end{proof}

Finally, we are ready to prove Theorem \ref{thm:problem_2_strict_saddle}.
\begin{proof}[Proof of Theorem \ref{thm:problem_2_strict_saddle}]

Similarly, $(\alpha,\gamma, \epsilon,\delta)$-\name immediately follows from Lemma \ref{lem:Problem2_case2} and Lemma \ref{lem:Problem2_case1}.

The only thing remains to show is that Optimization problem (\ref{eq:hardprob}) has exactly $2^d \cdot d!$ local minimum that corresponds to permutation and sign flips of $a_i$'s.
This can be easily proved by the same argument as in the proof of Theorem \ref{thm:problem_1_strict_saddle}.
\end{proof}

\end{document}